\documentclass{article}
\usepackage{ijcai26}

\usepackage{times}
\usepackage{soul}
\usepackage{url}
\usepackage[hidelinks]{hyperref}
\usepackage[utf8]{inputenc}
\usepackage[small]{caption}
\usepackage{graphicx}
\usepackage{amsmath}
\usepackage{amsthm}
\usepackage{amssymb}
\usepackage{booktabs}
\usepackage{algorithm}
\usepackage{algorithmic}
\usepackage{multirow}
\usepackage[switch]{lineno}
\usepackage[title]{appendix}
\usepackage{xcolor}

\usepackage{subcaption}   
\usepackage{float}        

\newtheorem{theorem}{Theorem}
\newtheorem{lemma}{Lemma}

\newtheorem{proposition}{Proposition}
\newcommand{\myparam}{\lambda}

\title{Temporal Smoothness Doubly Robust Learning for Debiased Knowledge Tracing}

\author{
Peilin Zhan\thanks{These authors contributed equally to this work.}\and
Wei Chen\footnotemark[1]\and
Weilin Chen\footnotemark[1]\and
Shuyi Pan\And
Ruichu Cai\thanks{Corresponding author.}\\
\affiliations
School of Computer Science and Technology, Guangdong University of Technology, Guangzhou, China\\
\emails
zhanpl01@outlook.com,
\{chenweidelight,
chenweilin.chn,
pansy0174,
cairuichu\}@gmail.com
}
 
\begin{document}

\maketitle

\begin{abstract}

    Knowledge Tracing (KT) is fundamental to intelligent education systems, yet relies on educational logs that are selectively observed. The non-random nature of exercise recommendations and student choices inevitably induces severe selection bias. Most existing KT methods neglect this issue, training on observed logs using standard empirical risk, which yields biased mastery estimates and accumulates errors in subsequent recommendations. To address this, we introduce a doubly robust (DR) formulation for KT that integrates a propensity model with an error imputation model, theoretically guaranteeing unbiasedness if either model is accurate. Beyond unbiasedness, in the sequential setting of KT, we identify that the estimator's performance is compromised by variance-dependent stochastic deviations that accumulate over time, thereby causing training instability and limiting performance. To mitigate this, we derive a generalization bound that explicitly characterizes the impact of estimator variance and identifies temporal smoothness as a key factor in controlling it. Building on these theoretical insights, we propose the Temporal Smoothness Doubly Robust (TSDR) framework. TSDR jointly optimizes the KT predictor and the imputation model with a smoothness regularizer, effectively reducing variance while preserving the unbiasedness guarantee of DR. Experiments on multiple real-world benchmarks demonstrate that TSDR consistently enhances various state-of-the-art KT backbones, underscoring the vital role of principled bias correction in KT. 
   
\end{abstract}

\section{Introduction}  


In intelligent education, tracing what students know is of great importance for adaptive learning ~\cite{ITS,ITS_effective}.
Knowledge Tracing (KT) serves this role by modeling students' dynamic knowledge states based on their historical learning interactions, aiming to predict how well they will master future questions.
However, despite the rapid progress of deep learning-based KT models ~\cite{DKT,sakt,AKT,sparsekt,dr4kt}, their predictions often exhibit systematic deviation from students' actual knowledge.
This discrepancy raises a fundamental question: are current KT models truly tracing students' knowledge, or just fitting their observed behaviors?

\begin{figure}[t] 
    \centering
    \includegraphics[width=0.5\textwidth]{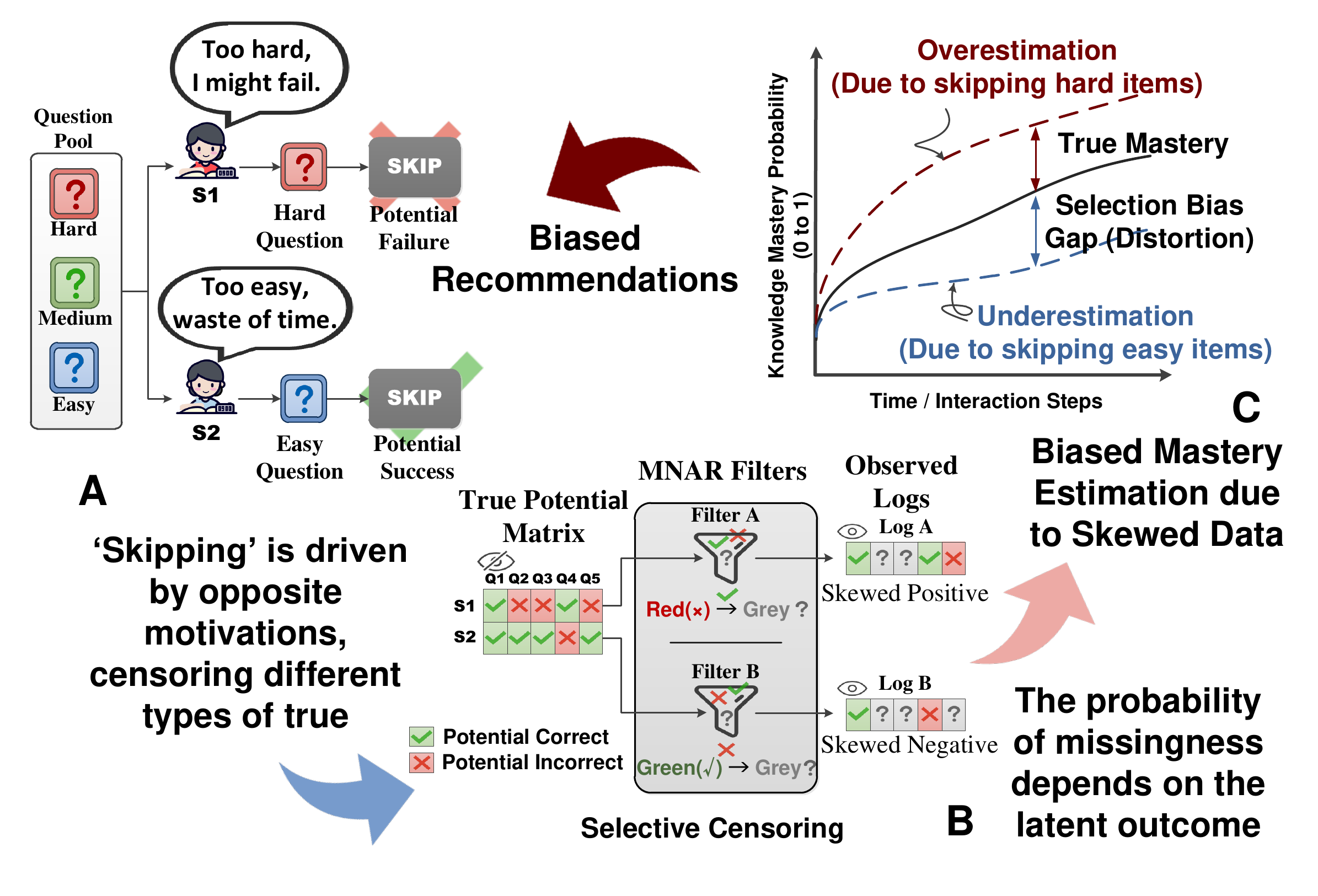}
    \caption{
        The causal mechanism of selection bias in KT: student skipping acts as a bias filter that systematically distorts the observed learning traces, causing models to misinterpret missingness and gradually internalize these distortions, ultimately learning biased knowledge states.
    }
    \label{fig:ill_bias}
\end{figure}

The reason for this discrepancy lies in the nature of the data on which current KT models are trained.
The recorded data exhibit bias since whether student-question interactions are observed may depend on the students' ability and the questions' difficulty, a phenomenon known as the Missing Not At Random (MNAR) property.
Taking Fig. \ref{fig:ill_bias} as an example, in online learning scenarios, students often strategically skip questions they deem ``too difficult" or ``boring" ~\cite{avoid1}. 
This behavior causes a lack of specific question types in the data.
Ignoring this missingness mechanism leads standard KT models to confound genuine knowledge states with interaction patterns, thereby introducing severe selection bias. 
This results in skewed proficiency estimates and diminished generalization capabilities. The issue is further compounded in adaptive systems, since recommendations are tailored to the student's estimated level via the Zone of Proximal Development (ZPD) ~\cite{ZPD}: biased estimates trigger biased recommendations, creating a feedback loop that exacerbates the bias~\cite{feedback_loop2,feedback_loop1}.

Addressing such a selection bias necessitates disentangling a student's actual performance from the observation probability.
This requires modeling two factors: which interactions are logged and what the responses would have been for unlogged interactions.
Accordingly, we adopt a Doubly Robust (DR) risk estimator that integrates selection correction via Inverse Propensity Weighting (IPW)~\cite{IPS1,IPS1_high_var} and potential performance estimation via Error Imputation (EIB)~\cite{EIB1}. 
This construction enjoys the double robustness property: the resulting risk estimate is unbiased as long as either the propensity model or the imputation model is accurate.
Nevertheless, directly applying DR in KT can be unstable: the training loss is prone to oscillation if imputation models are inaccurate~\cite{dr_j}. 
This issue is particularly critical in educational logs where responses are inherently noisy due to guessing and slipping~\cite{bkt_bias2008_1,aberrant2,hdkt_bias},
and such noise triggers error propagation and accumulation throughout deep sequential models, leading to significant error amplification.


To bridge this gap, we propose a Temporal Smoothness Doubly Robust (TSDR) learning framework for knowledge tracing. 
Rather than treating the instability as an inevitable side effect, we derive a generalization bound specifically for the KT context. 
This theoretical analysis reveals that the estimation risk is dominated by the instability of imputation errors. 
Guided by this insight, we design a joint learning paradigm where the KT predictor and the imputation model are co-optimized. 
Crucially, we introduce a temporal smoothness constraint to regularize the latent trajectory, which effectively tightens the generalization bound and suppresses variance, thereby preserving the unbiased nature of the DR estimator while ensuring training stability.
Overall, the main contributions of this paper are summarized as follows:

\begin{itemize}
\item We tackle the selection bias in KT induced by the non-random nature of exercise recommendations and student engagement choices. By adopting a causal inference perspective, we model the data generation process and propose a DR estimator to obtain unbiased mastery estimates.


\item We propose TSDR, a model-agnostic DR framework for bias correction that incorporates temporal smoothing to mitigate variance-related stochastic deviations. Supported by the proposed generalization bound, we employ a joint learning paradigm to co-optimize the KT and imputation models, ensuring robust performance.

\item 
We conduct extensive experiments on 9 real-world benchmark datasets. The empirical results validate the efficacy of our framework, showing that TSDR improves the performance of various state-of-the-art KT baselines. These findings highlight the practical value of addressing selection bias in knowledge tracing.
\end{itemize}


\section{Related Work}
\subsection{Classic Knowledge Tracing Models}
Early research in knowledge tracing is based on probabilistic graphical models, such as BKT~\cite{KT}. BKT leverages Hidden Markov Models to characterize knowledge state transitions, and Item Response Theory (IRT)~\cite{irt}, which focuses on modeling item difficulty and student ability. After that, the advent of deep learning catalyzed a profound paradigm shift. DKT ~\cite{DKT} pioneered the use of Recurrent Neural Networks to capture the hidden dynamics of knowledge acquisition within student interaction sequences, demonstrating strong adaptability in handling the complexities of real-world scenarios. Subsequently, the research paradigm gravitated towards Transformer-based architectures~\cite{transformer} for their superior ability to capture long-range dependencies. For instance, Self-Attentive Knowledge Tracing (SAKT)~\cite{sakt} first applied self-attention mechanisms to a student's historical interaction sequence to dynamically weigh the relevance of past exercises. AKT~\cite{AKT} further refined this by incorporating contextual information from the items themselves to mitigate interference from irrelevant exercises. More recent work has been dedicated to augmenting attention mechanisms to overcome their native inability to capture the dynamics of human forgetting or to robustly handle sequences of long lengths~\cite{stablekt_bias,folibikt_bias}. While their approaches differ, a significant body of work increasingly seeks to enhance model predictive performance from the perspective of learning high-quality representations for knowledge concepts or exercises. For instance, some studies have introduced Graph Neural Networks to model relationships among knowledge concepts~\cite{gkt}, while others have considered relations between knowledge concepts and exercises~\cite{mao1}, among other complex relationships.

However, these methods largely overlook systematic selection bias from non-random assignments, limiting their ability to generalize to counterfactual scenarios.
\subsection{Bias and Debiasing in Knowledge Tracing}
A central challenge in the development of robust Knowledge Tracing (KT) models is the mitigation of various biases that can compromise their validity. Existing research has predominantly focused on two key areas: first, addressing the inherent noise present within student interaction data; and second, preventing the model from learning spurious correlations during the training process.
\paragraph{Debiasing from Noise in Student Responses.}
The first major line of research targets the inherent noise in student response data, which can distort the model's inference of a student's true knowledge state. This noise often stems from aberrant student behaviors. Early models explicitly incorporated parameters for guessing and slipping to account for such behaviors~\cite{bkt_bias2024_2}. 
More recent work has explored nuanced modeling of factors like student disengagement~\cite{aberrant1}
and added small perturbations to input embeddings~\cite{adversarial_aberrant3} to improve robustness. State-of-the-art approaches now employ adversarial training and contrastive learning to generate debiased knowledge representations that are resilient to abnormal responses, thereby isolating the most reliable signals for prediction~\cite{cr_debiased,hdkt_bias}, also acknowledging that student actions do not always directly reflect their knowledge~\cite{cr_debiased}. While mitigating response noise, they focus on denoising rather than data recovery, failing to account for selection bias from non-random missingness.


\paragraph{Addressing Spurious Correlations and Causal Confusion.}
The second major area of research addresses the risk of the model learning spurious correlations from confounding factors, which is often termed ``causal confusion" \cite{causal_confusion}. This occurs when a model learns shortcuts from data artifacts instead of isolating the student's true knowledge. These confounders can stem from various sources, including the model's own architectural limitations. To counteract confounding more comprehensively, researchers have utilized a spectrum of techniques, many of which are predicated on the principles of causal inference. For instance, to address confounding from the imbalanced informativeness of historical responses, DR4KT introduces a reweighting framework~\cite{dr4kt}. Other approaches use formal causal modeling to disentangle effects. The CORE framework mitigates answer bias by modeling and subtracting the direct causal effect of the question on the response~\cite{core}. Similarly, DisKT uses a disentangled modeling approach and a causal subtraction mechanism to remove confounding from uneven performance distributions across concepts~\cite{diskt_bias}. For more complex scenarios involving unobserved confounders, recent work has leveraged the front-door criterion from causal inference, implementing it through causal self-attention mechanisms to identify the true, unconfounded pathway from knowledge to performance~\cite{front_door_kt}. While recent causal-aware approaches attempt to mitigate specific confounders, they primarily focus on feature-level disentanglement rather than correcting the systemic selection bias in the data generation process. 

\subsection{MNAR and DR Debiasing}
While existing KT approaches address biases regarding \textit{what} is recorded and \textit{how} it is modeled, they largely overlook a fundamental bias rooted in \textit{why} interactions are recorded: selection bias. This challenge is closely related to the MNAR problem in Recommender Systems (RS), where doubly robust learning mitigates selection bias by combining propensity weighting with error imputation~\cite{dr1,dr_j}. Beyond RS, DR has also been explored in structured observational settings such as networked interference~\cite{tnet}. Similar concerns arise in time-series imputation, where causal perspectives emphasize missing mechanisms in temporal data~\cite{cai2025causalview}. However, KT differs from these settings by modeling evolving knowledge states rather than static preferences, networked effects, or imputed values. Directly applying existing DR estimators to KT may therefore cause high-variance corrections and error accumulation along the learning sequence. We thus adapt DR to KT and introduce temporal smoothness to stabilize sequential training.

\section{Problem Formulation}
The objective of KT is to predict a student's performance on an upcoming question by modeling their knowledge state from past interactions. Without loss of generality, we define a student's historical interaction sequence as $Seq_t := \{ (c_1, r_1), (c_2, r_2), \dots, (c_t, r_t) \}$, where $c_t$ and $r_t$ denote 
the concept from the concept set $\mathcal{C}$ and the binary response correctness ($r_t \in \{0, 1\}$) at time step $t$ respectively. Given the next 
concept $c_{t+1}$, the KT model aims to first estimate the student's latent knowledge state $\hat{h}_t = f_\theta(Seq_t)$, and then use this state to predict the probability of a correct answer, $\hat{r}_{t+1} = g_\phi(\hat{h}_t,  c_{t+1})$. The model parameters $\theta$ and $\phi$ are typically optimized by minimizing the prediction error, e.g., binary cross-entropy (BCE), on the observed data.

Since interactions are selectively acquired (MNAR) rather than random based on recommendation policies or student preferences, minimizing error on observed data results in a biased estimator. We typically analyze this bias and derive a causal-aware solution to mitigate it, given the students' historical interaction sequences data.

\section{Theoretical Analysis}
\label{sec:theoretical_estimation}
To address the selection bias identified in the problem formulation, we reframe the KT task from the perspective of counterfactual estimation ~\cite{pearl2016causal} for model performance. We formalize the discrepancy between the ideal true risk and the naive empirical risk, and then derive the DR estimator as a principled solution.

\subsection{Revisiting Estimation Problem}
Let data $\mathcal{D}_\text{full} = \{(t, c) \mid 1 \le t \le T, c \in \mathcal{C}\}$ with $T$ being the maximum length of the interaction sequence denote the ideal full set of potential interactions for a student sequence. We define the true risk $\mathcal{R}_{true}$ as the average prediction error over this full distribution of data:
\begin{equation}
    \mathcal{R}_{true} = \mathbb{E}_{(t,c) \sim \mathcal{D}_\text{full}} [e_{t+1,c}] = \frac{1}{N} \sum_{t=1}^{T-1} \sum_{c=1}^{|\mathcal{C}|} e_{t+1,c},
\end{equation}
where $N = (T-1) \times |\mathcal{C}|$ represents the total number of prediction targets, and $e_{t+1,c} = - \big( r_{t+1,c} \log(\hat{r}_{t+1,c}) + (1 - r_{t+1,c}) \log(1 - \hat{r}_{t+1,c}) \big)$ is the BCE loss for skill $c$ at time $t$.

In practice, we only observe a subset of interactions where the observation indicator $o_{t+1,c}=1$. The \textbf{Naive Estimator}, which corresponds to the standard KT objective, calculates the risk solely on observed data:
\begin{equation}
    \hat{\mathcal{R}}_{naive} = \frac{\sum o_{t+1,c} \cdot e_{t+1,c}}{\sum o_{t+1,c}}.
\end{equation}
Since the observation probability, i.e., propensity $p_{t+1,c} = P(o_{t+1,c}=1|h_{t},c)$, is non-uniform due to the MNAR mechanism, $\mathbb{E}[\hat{\mathcal{R}}_{naive}] \neq \mathcal{R}_{true}$. This estimation bias misguides the optimization, causing the model to overfit to the specific selection policy rather than learning the true knowledge dynamics.




\subsection{Constructing DR Estimator}
To obtain an unbiased estimate of $\mathcal{R}_\text{true}$, we combine two distinct approaches: the imputation-based method modeling $\hat{e}_{t+1,c}$ and the propensity-based method modeling $\hat{p}_{t+1,c}$.

We introduce the DR estimator, denoted as $\hat{\mathcal{R}}_\text{DR}$. Instead of viewing it merely as a loss function, we define it as a statistical estimator for the unobserved true risk:
\begin{equation} \label{eq:dr}
    \hat{\mathcal{R}}_\text{DR} = \frac{1}{N} \sum_{t=1}^{T-1} \sum_{c=1}^{|\mathcal{C}|} \bigg[ \underbrace{\hat{e}_{t+1,c}}_{\text{Imputation}} + \underbrace{\frac{o_{t+1,c}}{\hat{p}_{t+1,c}} (e_{t+1,c} - \hat{e}_{t+1,c})}_{\text{Correction Term}} \bigg],
\end{equation}
where $N$ represents the total number of entries, $T$ is the maximum sequence length, and $\mathcal{C}$ denotes the set of knowledge concepts. The intuition behind this formulation is to leverage the imputed error $\hat{e}_{t+1,c}$ as a \textbf{baseline estimate} for the risk on the full distribution. 
Since the imputation model alone may be biased, the second term, $\frac{o_{t+1,c}}{\hat{p}_{t+1,c}} (e_{t+1,c} - \hat{e}_{t+1,c})$, acts as a \textbf{bias correction}. It adjusts the baseline using the residual between the true error $e_{t+1,c}$ and the imputed error $\hat{e}_{t+1,c}$ for observed interactions, weighted by the inverse propensity. This structure explicitly demonstrates the ``double robustness'' property: if the imputation is accurate (i.e., $\hat{e} \approx e$), the correction term vanishes, effectively minimizing variance; conversely, if the propensity model is accurate, the correction term ensures the estimator remains unbiased even if the baseline imputation is poor.

\subsection{Theoretical Guarantee from Bias Analysis}
To show why this estimator is preferred, we analyze its bias with respect to the true risk. The bias is defined as the difference between the expected value of the estimator over the observation distribution $O$ and the true target. The complete theoretical framework and main results are presented in Appendix \ref{sec:theoretical_main}, with detailed proofs provided in Appendix \ref{sec:detailed_proofs}. 

For clearly, we define the imputation error as $\delta_{t,c} := e_{t,c} - \hat{e}_{t,c}$ and the propensity error as $\Delta_{t,c} := 1- \frac{p_{t,c}}{\hat{p}_{t,c}} $.

\begin{lemma}[Bias of the DR Estimator] \label{lmm: bias}
The bias of the DR estimator is bounded by the product of the errors of the two auxiliary models:
\begin{equation}
    \left| \mathbb{E}_O [\hat{\mathcal{R}}_\text{DR}] - \mathcal{R}_\text{true} \right| \leq \frac{1}{N} \sum_{(t,c)\in\mathcal{D}_\text{full}} \left|  \Delta_{t+1,c} \cdot \delta_{t+1,c}\right|.
\end{equation}
\end{lemma}

\begin{proof}
\textit{(Sketch)} By linearity of expectation, $\mathbb{E}_O [o_{t,c}] = p_{t,c}$. Substituting this into the expectation of $\hat{\mathcal{R}}_\text{DR}$:
\begin{equation}
\begin{aligned}
    \mathbb{E}_O [\hat{\mathcal{R}}_\text{DR}] &= \frac{1}{N} \sum_{t,c} \left( \hat{e}_{t+1,c} + \frac{p_{t+1,c}}{\hat{p}_{t+1,c}}(e_{t+1,c} - \hat{e}_{t+1,c}) \right) \\
    &= \frac{1}{N} \sum_{t,c} \left( e_{t+1,c} \frac{p_{t+1,c}}{\hat{p}_{t+1,c}} + \hat{e}_{t+1,c} (1 - \frac{p_{t+1,c}}{\hat{p}_{t+1,c}}) \right).
\end{aligned}
\end{equation}
Subtracting $\mathcal{R}_{true} = \frac{1}{N} \sum e_{t+1,c}$ from this expectation yields the bias term proportional to $\delta_{t+1,c} \cdot \Delta_{t+1,c}$.
\end{proof}

Lemma \ref{lmm: bias} confirms the \textbf{Double Robustness} property: the estimation is unbiased if \textbf{either} the propensity model or the imputation model is accurate. 
More attractively, the bias depends on a \textbf{product} of the two model errors.
When both models are learned with sufficiently expressive estimators (e.g., neural networks) and their estimation errors can be controlled, the DR estimator achieves a \textbf{second-order error decay}, resulting in an error that is smaller than that of estimators typically dominated by a single estimation error.

\subsection{Generalization Bound and Variance Control}
While the DR estimator theoretically guarantees unbiasedness under mild conditions, its practical application on finite interaction sequences is often hindered by stochastic fluctuations. To rigorously quantify this risk, we derive a generalization bound that accounts for both the bias and the concentration properties of the estimator.

\begin{theorem}[Generalization Bound]\label{thm:generalization_bound}
For any finite hypothesis space $\mathcal{H}$, with probability at least $1-\eta$, the true risk of the optimal prediction is bounded by:
\begin{equation}
\begin{aligned}
    \mathcal{R}_{true} \leq \underbrace{\hat{\mathcal{R}}_\text{DR}}_{\text{Empirical Risk}}  &+ \underbrace{\frac{1}{N}\sum_{(t,c)\in\mathcal{D}_\text{full}} |\Delta_{t+1,c}\delta_{t+1,c}|}_{\text{Bias Term}} \\
    &+ \underbrace{\sqrt{\frac{\ln(2|\mathcal{H}|/\eta)}{2N^2} \sum_{(t,c)\in\mathcal{D}_\text{full}} \left( \frac{\delta_{t+1,c}}{\hat{p}_{t+1,c}} \right)^2}}_{\text{Variance Term}}.
\end{aligned}
\end{equation}
\end{theorem}

\begin{proof}
(Sketch) The bound is derived by decomposing the risk into bias and variance components. The bias term follows from Lemma \ref{lmm: bias}. The variance term is bounded using the Azuma-Hoeffding inequality~\cite{azuma1967weighted}, accounting for the sequential nature of student interactions, combined with a union bound over $\mathcal{H}$.  
\end{proof}
Theorem \ref{thm:generalization_bound} reveals a critical insight: the total generalization error is dominated not just by the bias product but also by the Variance Term, which scales with the squared imputation error $\delta_{t,c}^2$. 
Inaccuracy in the imputation model results in large error deviations $\delta_{u,i}$, which loosens the generalization bound and in turn leads to unstable training.
To minimize $\delta_{t,c}$ and tighten the generalization bound, we introduce a temporal smoothness constraint (Section~\ref{eq:ts}).
This constraint works by regularizing the latent trajectory.
Furthermore, we adopt a joint learning approach to couple the optimization processes (Section~\ref{sec:joint_learning}).



\section{Method}
\label{sec:Method}
We re-conceptualize the challenge of debiasing a temporal learning process, treating the knowledge state at each distinct time step as a temporal “user” and operating under the principle that each momentary state possesses its own distinct learning preferences and biases. Adopting this perspective allows us to apply a DR strategy to learn from these state-specific interactions in a robust manner.

\begin{figure*}[ht]          
  \centering
  \includegraphics[width=0.86\textwidth]{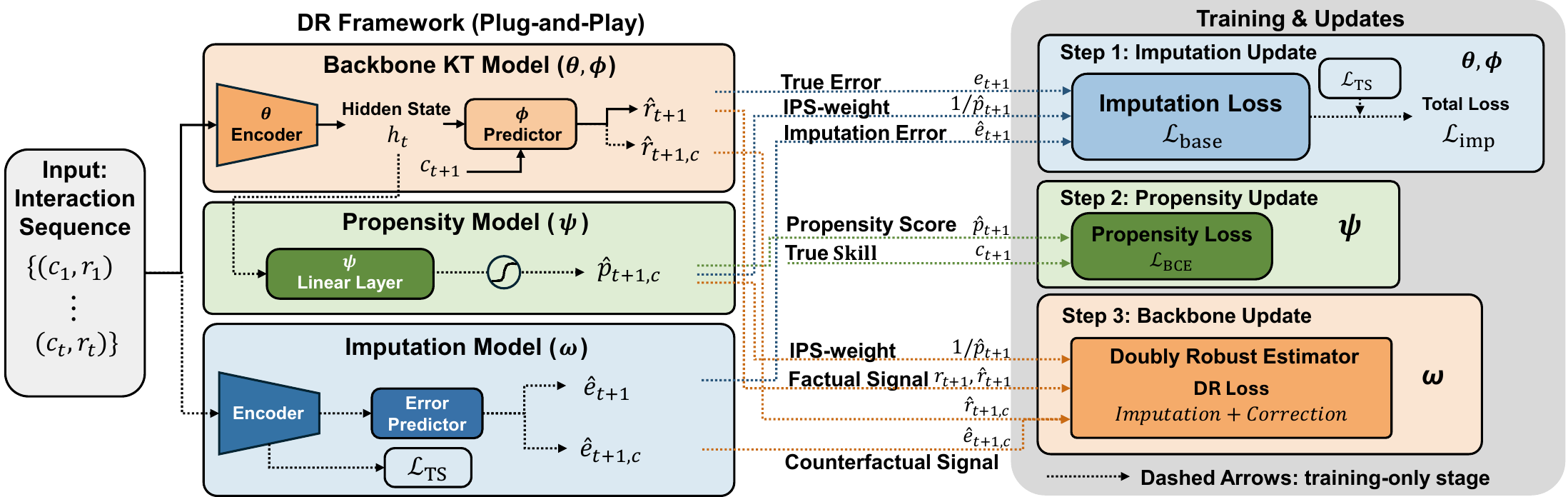}
   \caption{
        A Model-Agnostic DR Framework Overview.
    }
  \label{fig:Architecture}
\end{figure*}

\subsection{Model Architecture}
We propose a \textbf{Temporal-Smooth Doubly Robust (TSDR)} learning framework that jointly learns a propensity model, an error imputation model, and a KT model. The overall architecture of our framework is presented in Figure~\ref{fig:Architecture}.

\subsubsection{Propensity Model for Interaction Selection}
We learn the propensity score to model the interaction probability as: 
\begin{equation}
    \hat{p}_{t+1,c} = P(o_{t+1,c}=1 \mid \hat{h}_t, c).
\end{equation}
 
We model this as a \textbf{Multivariate Bernoulli distribution} for all concepts using a sigmoid output layer. The model is trained via BCE loss, treating unobserved concepts as implicit negatives.

\subsubsection{Error Imputation Model for Counterfactual Prediction}
The second component is the error imputation model, parameterized by $\omega$. Unlike the prediction model that focuses on knowledge mastery, this module aims to capture the dynamics of prediction residuals.
To decouple these distinct objectives, we employ a separate encoder for the imputation task. Let $h_t^{imp}$ denote the imputation-specific latent state derived from the interaction sequence $Seq_t$:
\begin{equation}
    h_t^{imp} = \text{Enc}_{\omega}(Seq_t) ,
\end{equation}
where $\text{Enc}_{\omega}$ denotes the internal knowledge state extractor of the imputation model, similar to the KT backbone.
The imputed error $\hat{e}_{t+1}$ is then predicted based on this state and a potential next concept:
\begin{equation}
    \hat{e}_{t+1,c} = g_{\omega}(h_t^{imp}, c).
\end{equation}

\paragraph{Optimization Objective.}The training of the imputation model consists of two components. First, the model aims to regress the true prediction errors $e_{t+1}$ based on the observed data. To address the distribution shift, we adopt an Inverse Propensity Score (IPS) weighted Mean Squared Error (MSE) as the base objective:
\begin{equation}
    \mathcal{L}_{\text{base}}(\omega) = \frac{1}{T-1}\sum_{t=1}^{T-1} \frac{(\hat{e}_{t+1} - e_{t+1})^2}{\hat{p}_{t+1}} ,
\end{equation}
where we focus solely on observation loss. 
Then, guided by the theoretical analysis in Theorem \ref{thm:generalization_bound}, we incorporate a \textbf{Temporal Smoothness} constraint on the latent states $h_t^{imp}$. 
This aligns with the Power Law of Practice~\cite{learning_curve}, which characterizes skill acquisition as a smooth, continuous power-function trajectory rather than a series of abrupt changes. 
By explicitly restricting the path length of $H^{imp}$, this constraint reduces the Rademacher complexity and tightens the generalization bound. 
The total loss function is formulated as:
\label{eq:ts}
\begin{equation}
    \mathcal{L}_{\text{imp}}(\omega) = \mathcal{L}_{\text{base}}(\omega) + \frac{\lambda}{T-1}\sum_{t=1}^{T-1} (h_t^{\text{imp}} - h_{t-1}^{\text{imp}})^2 ,
\end{equation}
where $\lambda$ is the coefficient controlling the strength of the smoothness regularization.

\subsubsection{KT Model with Doubly Robust Estimator}
As defined in our problem formulation, the core of our approach is the KT model, which consists of an encoder $f_\theta$ and a predictor $g_\phi$. The encoder $f_\theta(Seq_t)$ processes the interaction history to produce the knowledge state $\hat{h}_t$. The predictor $g_\phi(\hat{h}_t, c_{t+1})$ then outputs the probability of a correct answer $\hat{r}_{t+1,c}$. Our ultimate goal is to learn the parameters $\theta$ and $\phi$ in an unbiased manner. The vanilla prediction loss for a single step is given by the BCE:
\begin{equation}
    e_{t+1} = -  r_{t+1} \log(\hat{r}_{t+1}) - (1 - r_{t+1}) \log(1 - \hat{r}_{t+1}) .
\end{equation}

\paragraph{Doubly Robust Estimator.}
Adhering to the principle of double robustness, we train the KT model by minimizing the doubly robust loss instead of the naive BCE loss. As restated in Eq.~(\ref{eq:dr}), the loss is given by:
\begin{equation}
    \mathcal{L}_\text{DR}(\theta, \phi) = \frac{1}{N} \sum_{t=1}^{T-1}\sum_{c\in \mathcal{C}} \left( \hat{e}_{t+1,c} + \frac{o_{t+1,c}}{\hat{p}_{t+1,c}} (e_{t+1,c} - \hat{e}_{t+1,c}) \right) .
\end{equation}


\begin{algorithm}[ht]
    \caption{Joint Learning with TSDR}
    \label{alg:joint_learning}
    { 
    \raggedright
    \textbf{Input}: Observed interaction data $\mathcal{D}_{obs}$, Concept set $\mathcal{C}$.\\
    \textbf{Parameters}: Learning rates $\alpha$.\\
    \textbf{Output}: Trained KT model parameters $\theta, \phi$ \\ 
    } 
    \begin{algorithmic}[1] 
        \STATE Initialize parameters $\theta, \phi, \psi, \omega$.
        \WHILE{stopping criteria not met}
            \STATE Sample a mini-batch $\mathcal{B}$ from $\mathcal{D}_{obs}$.
            
            \STATE /* Phase 1: Update Auxiliary Models */
            \STATE Update imputation model: $\omega \leftarrow \omega - \alpha \nabla_\omega \mathcal{L}_{imp}(\omega)$
            \STATE Update propensity model: $\psi \leftarrow \psi - \alpha \nabla_\psi \mathcal{L}_{prop}(\psi)$

            \STATE /* Phase 2: Update the KT Model */
            \STATE Based on observed history in $\mathcal{B}$, compute counterfactual errors $\hat{e}_{t+1,c}$ for all concepts $c \in \mathcal{C}$.
            \STATE Calculate $\mathcal{L}_\text{DR}(\theta, \phi)$ by summing errors over $\mathcal{C}$.
            \STATE $\theta, \phi \leftarrow (\theta, \phi) - \alpha \nabla_{\theta,\phi} \mathcal{L}_\text{DR}(\theta, \phi)$
        \ENDWHILE
        \STATE \textbf{return} Optimized KT parameters $\theta, \phi$.
    \end{algorithmic}
\end{algorithm}
\subsection{Joint Learning Approach} 
\label{sec:joint_learning}
To mitigate the impact of imputation errors on prediction, we derive a generalization bound (Appendix~\ref{sec:detailed_proofs}) following~\cite{dr_j}. Motivated by this theoretical insight, we co-optimize the parameters ($\theta, \phi, \psi, \omega$) via an alternating strategy, as detailed in Algorithm~\ref{alg:joint_learning}.
This joint procedure allows the models to mutually regularize each other, leading to a more robust estimation of the true prediction inaccuracy and ultimately, a more effective and unbiased KT model.

\section{Experiment}
In this section, 
our evaluation aims to answer the following research questions:
\begin{itemize}
    \item \textbf{RQ1:} Does TSDR consistently enhance the performance of SOTA baselines on real-world datasets?
    \item \textbf{RQ2:} Can TSDR mitigate performance degradation under varying MNAR levels in controlled simulations?
    \item \textbf{RQ3:} How sensitive is the model to the intensity of the temporal smoothness constraint?
\end{itemize}

\subsection{Experimental Setup}

\paragraph{Benchmark Dataset.}We evaluate the proposed framework on 9 publicly available benchmark datasets covering diverse domains: Spanish\footnote{\url{https://github.com/robert-lindsey/WCRP}}~\cite{spanish}, ASSISTments17\footnote{\url{https://sites.google.com/view/assistmentsdatamining/dataset}}~\cite{assistments17}, Slepemapy\footnote{\url{ttps://www.fi.muni.cz/adaptivelearning/?a=data}}~\cite{slepemapy}, Algebra05\footnote{\url{https://pslcdatashop.web.cmu.edu/KDDCup/}}~\cite{algebra}, [Prob, Linux, Database, Comp]\footnote{\url{https://github.com/wahr0411/PTADisc}}~\cite{ptadisc} and EdNet\footnote{\url{https://github.com/riiid/ednet}}~\cite{ednet}. These datasets vary significantly in scale and density, offering a comprehensive testbed for robustness. Detailed statistics are available in Appendix~\ref{app:benchmark}.

\paragraph{Synthetic Dataset.}
To evaluate robustness against the bias, we generated 6 synthetic datasets using a controllable MNAR mechanism with bias degrees $\gamma \in [0, 0.999]$. Generation details and statistics are provided in Appendix~\ref{app:syn_test}.

\paragraph{Baselines.}To validate TSDR as a generic strategy, we apply it to 6 SOTA baselines (DKT~\cite{DKT}, AKT~\cite{AKT}, simpleKT~\cite{simplekt}, FoLiBiKT~\cite{folibikt_bias}, SparseKT~\cite{sparsekt}, and DisKT~\cite{diskt_bias}) and compare them against their original versions. Details of baselines are provided in Appendix~\ref{app:baselines}.

\begin{table*}[t]
    \centering
    \resizebox{\textwidth}{!}{ 
        \begin{tabular}{llcccccc}
            \toprule
            \textbf{Dataset} & \textbf{Metric} & \textbf{DKT}& \textbf{AKT}& \textbf{simpleKT}& \textbf{FoLiBiKT}& \textbf{SparseKT}& \textbf{DisKT}\\
            \midrule
            
            \multirow{3}{*}{\textbf{Spanish}} 
            & AUC $\uparrow$ & 0.7798 / \textbf{0.8018} {\footnotesize ($\uparrow$ 2.82\%)}
                            & 0.7967 / \textbf{0.8034} {\footnotesize ($\uparrow$ 0.84\%)}
                            & 0.8098 / \textbf{0.8188} {\footnotesize ($\uparrow$ 1.11\%)}
                            & 0.7992 / \textbf{0.8067} {\footnotesize ($\uparrow$ 0.94\%)}
                            & 0.8019 / \textbf{0.8197} {\footnotesize ($\uparrow$ 2.22\%)}
                            & 0.8243 / \textbf{0.8259} {\footnotesize ($\uparrow$ 0.19\%)} \\
            & ACC $\uparrow$ & 0.7250 / \textbf{0.7362} {\footnotesize ($\uparrow$ 1.54\%)} 
                            & 0.7386 / \textbf{0.7420} {\footnotesize ($\uparrow$ 0.46\%)}
                            & 0.7405 / \textbf{0.7494} {\footnotesize ($\uparrow$ 1.20\%)}
                            & 0.7332 / \textbf{0.7388} {\footnotesize ($\uparrow$ 0.76\%)} 
                            & 0.7324 / \textbf{0.7510} {\footnotesize ($\uparrow$ 2.54\%)}
                            & 0.7520 / \textbf{0.7582} {\footnotesize ($\uparrow$ 0.82\%)} \\
            & RMSE $\downarrow$ & 0.4301 / \textbf{0.4195} {\footnotesize ($\downarrow$ 2.46\%)} 
                            & 0.4294 / \textbf{0.4241} {\footnotesize ($\downarrow$ 1.23\%)} 
                            & 0.4214 / \textbf{0.4186} {\footnotesize ($\downarrow$ 0.66\%)}
                            & 0.4292 / \textbf{0.4264} {\footnotesize ($\downarrow$ 0.65\%)}
                            & 0.4319 / \textbf{0.4204} {\footnotesize ($\downarrow$ 2.66\%)}
                            & 0.4199 / \textbf{0.4096} {\footnotesize ($\downarrow$ 2.45\%)} \\
            \midrule

            \multirow{3}{*}{\textbf{ASSISTments17}} 
            & AUC & 0.6358 / \textbf{0.6472} {\footnotesize ($\uparrow$ 1.79\%)} 
                  & 0.6843 / \textbf{0.7012} {\footnotesize ($\uparrow$ 2.47\%)} 
                  & 0.6692 / \textbf{0.6905} {\footnotesize ($\uparrow$ 3.18\%)} 
                  & 0.6826 / \textbf{0.6994} {\footnotesize ($\uparrow$ 2.46\%)} 
                  & 0.6800 / \textbf{0.7124} {\footnotesize ($\uparrow$ 4.76\%)} 
                  & 0.7223 / \textbf{0.7324} {\footnotesize ($\uparrow$ 1.40\%)} \\
            & ACC & 0.6172 / \textbf{0.6249} {\footnotesize ($\uparrow$ 1.25\%)} 
                  & 0.6525 / \textbf{0.6617} {\footnotesize ($\uparrow$ 1.41\%)} 
                  & 0.6430 / \textbf{0.6590} {\footnotesize ($\uparrow$ 2.49\%)} 
                  & 0.6538 / \textbf{0.6619} {\footnotesize ($\uparrow$ 1.24\%)} 
                  & 0.6442 / \textbf{0.6693} {\footnotesize ($\uparrow$ 3.90\%)}
                  & 0.6621 / \textbf{0.6784} {\footnotesize ($\uparrow$ 2.46\%)} \\
            & RMSE & 0.4807 / \textbf{0.4774} {\footnotesize ($\downarrow$ 0.69\%)}
                  & 0.4762 / \textbf{0.4651} {\footnotesize ($\downarrow$ 2.33\%)}
                  & 0.4791 / \textbf{0.4682} {\footnotesize ($\downarrow$ 2.28\%)} 
                  & 0.4750 / \textbf{0.4659} {\footnotesize ($\downarrow$ 1.92\%)}
                  & 0.4767 / \textbf{0.4624} {\footnotesize ($\downarrow$ 3.00\%)}
                  & 0.4740 / \textbf{0.4641} {\footnotesize ($\downarrow$ 2.09\%)} \\
            \midrule

            \multirow{3}{*}{\textbf{Slepemapy}} 
            & AUC & 0.6806 / \textbf{0.6937} {\footnotesize ($\uparrow$ 1.92\%)} 
                  & 0.7127 / \textbf{0.7208} {\footnotesize ($\uparrow$ 1.14\%)} 
                  & 0.7108 / \textbf{0.7143} {\footnotesize ($\uparrow$ 0.49\%)} 
                  & 0.7120 / \textbf{0.7199} {\footnotesize ($\uparrow$ 1.11\%)} 
                  & 0.7478 / \textbf{0.7758} {\footnotesize ($\uparrow$ 3.74\%)} 
                  & 0.7283 / \textbf{0.7345} {\footnotesize ($\uparrow$ 0.85\%)} \\
            & ACC & 0.7740 / \textbf{0.7795} {\footnotesize ($\uparrow$ 0.71\%)} 
                  & 0.7796 / \textbf{0.7833} {\footnotesize ($\uparrow$ 0.47\%)} 
                  & 0.7743 / \textbf{0.7830} {\footnotesize ($\uparrow$ 1.12\%)} 
                  & 0.7799 / \textbf{0.7843} {\footnotesize ($\uparrow$ 0.56\%)} 
                  & 0.7845 / \textbf{0.7944} {\footnotesize ($\uparrow$ 1.26\%)} 
                  & 0.7805 / \textbf{0.7861} {\footnotesize ($\uparrow$ 0.72\%)} \\
            & RMSE & 0.4021 / \textbf{0.3977} {\footnotesize ($\downarrow$ 1.09\%)} 
                  & 0.3948 / \textbf{0.3915} {\footnotesize ($\downarrow$ 0.84\%)} 
                  & 0.3969 / \textbf{0.3927} {\footnotesize ($\downarrow$ 1.06\%)} 
                  & 0.3962 / \textbf{0.3915} {\footnotesize ($\downarrow$ 1.19\%)} 
                  & 0.3866 / \textbf{0.3768} {\footnotesize ($\downarrow$ 2.53\%)} 
                  & 0.3901 / \textbf{0.3886} {\footnotesize ($\downarrow$ 0.38\%)} \\
            \midrule

            \multirow{3}{*}{\textbf{Algebra05}} 
            & AUC & 0.7538 / \textbf{0.7669} {\footnotesize ($\uparrow$ 1.74\%)}
                  & 0.7698 / \textbf{0.7797} {\footnotesize ($\uparrow$ 1.29\%)}
                  & 0.7677 / \textbf{0.7809} {\footnotesize ($\uparrow$ 1.71\%)} 
                  & 0.7718 / \textbf{0.7781} {\footnotesize ($\uparrow$ 0.82\%)}
                  & 0.7677 / \textbf{0.7772} {\footnotesize ($\uparrow$ 1.24\%)}
                  & 0.7769 / \textbf{0.7862} {\footnotesize ($\uparrow$ 1.20\%)} \\
            & ACC & 0.7748 / \textbf{0.7852} {\footnotesize ($\uparrow$ 1.34\%)} 
                  & 0.7794 / \textbf{0.7825} {\footnotesize ($\uparrow$ 0.40\%)}
                  & 0.7734 / \textbf{0.7834} {\footnotesize ($\uparrow$ 1.29\%)} 
                  & 0.7813 / \textbf{0.7857} {\footnotesize ($\uparrow$ 0.56\%)}
                  & 0.7694 / \textbf{0.7876} {\footnotesize ($\uparrow$ 2.37\%)}
                  & 0.7817 / \textbf{0.7912} {\footnotesize ($\uparrow$ 1.22\%)} \\
            & RMSE & 0.4013 / \textbf{0.3895} {\footnotesize ($\downarrow$ 2.94\%)} 
                  & 0.3955 / \textbf{0.3937} {\footnotesize ($\downarrow$ 0.46\%)}
                  & 0.3980 / \textbf{0.3901} {\footnotesize ($\downarrow$ 1.98\%)}
                  & 0.3911 / \textbf{0.3904} {\footnotesize ($\downarrow$ 0.18\%)}
                  & 0.3954 / \textbf{0.3891} {\footnotesize ($\downarrow$ 1.59\%)} 
                  & 0.3958 / \textbf{0.3886} {\footnotesize ($\downarrow$ 1.82\%)} \\
            \midrule

            \multirow{3}{*}{\textbf{Prob}} 
            & AUC & 0.7278 / \textbf{0.7319} {\footnotesize ($\uparrow$ 0.56\%)}
                  & 0.7347 / \textbf{0.7719} {\footnotesize ($\uparrow$ 5.06\%)} 
                  & 0.7447 / \textbf{0.7716} {\footnotesize ($\uparrow$ 3.61\%)} 
                  & 0.7409 / \textbf{0.7703} {\footnotesize ($\uparrow$ 3.97\%)} 
                  & 0.7597 / \textbf{0.7775} {\footnotesize ($\uparrow$ 2.34\%)} 
                  & 0.7617 / \textbf{0.7870} {\footnotesize ($\uparrow$ 3.32\%)} \\
            & ACC & 0.6770 / \textbf{0.6822} {\footnotesize ($\uparrow$ 0.77\%)} 
                  & 0.6916 / \textbf{0.7136} {\footnotesize ($\uparrow$ 3.18\%)} 
                  & 0.6956 / \textbf{0.7110} {\footnotesize ($\uparrow$ 2.21\%)} 
                  & 0.6956 / \textbf{0.7126} {\footnotesize ($\uparrow$ 2.44\%)} 
                  & 0.7062 / \textbf{0.7197} {\footnotesize ($\uparrow$ 1.91\%)} 
                  & 0.7023 / \textbf{0.7262} {\footnotesize ($\uparrow$ 3.40\%)} \\
            & RMSE & 0.4527 / \textbf{0.4510} {\footnotesize ($\downarrow$ 0.38\%)} 
                  & 0.4527 / \textbf{0.4407} {\footnotesize ($\downarrow$ 2.65\%)} 
                  & 0.4556 / \textbf{0.4466} {\footnotesize ($\downarrow$ 1.98\%)}
                  & 0.4501 / \textbf{0.4403} {\footnotesize ($\downarrow$ 2.18\%)}
                  & 0.4475 / \textbf{0.4441} {\footnotesize ($\downarrow$ 0.76\%)} 
                  & 0.4460 / \textbf{0.4364} {\footnotesize ($\downarrow$ 2.15\%)} \\
            \midrule

            \multirow{3}{*}{\textbf{Linux}} 
            & AUC & 0.7477 / \textbf{0.7521} {\footnotesize ($\uparrow$ 0.59\%)}
                  & \textbf{0.8164} / 0.8162 {\footnotesize ($\downarrow$ 0.02\%)}
                  & 0.8151 / \textbf{0.8162} {\footnotesize ($\uparrow$ 0.13\%)}
                  & 0.8162 / \textbf{0.8163} {\footnotesize ($\uparrow$ 0.01\%)}
                  & 0.8315 / \textbf{0.8413} {\footnotesize ($\uparrow$ 1.18\%)}
                  & 0.8269 / \textbf{0.8278} {\footnotesize ($\uparrow$ 0.11\%)}\\
            & ACC & 0.7690 / \textbf{0.7708} {\footnotesize ($\uparrow$ 0.23\%)}
                  & 0.7975 / \textbf{0.7979} {\footnotesize ($\uparrow$ 0.05\%)}
                  & 0.7963 / \textbf{0.7980} {\footnotesize ($\uparrow$ 0.21\%)}
                  & 0.7970 / \textbf{0.7976} {\footnotesize ($\uparrow$ 0.08\%)}
                  & 0.8039 / \textbf{0.8084} {\footnotesize ($\uparrow$ 0.56\%)}
                  & 0.8045 / \textbf{0.8070} {\footnotesize ($\uparrow$ 0.31\%)}\\
            & RMSE & 0.4009 / \textbf{0.3996} {\footnotesize ($\downarrow$ 0.32\%)} 
                  & 0.3760 / \textbf{0.3750} {\footnotesize ($\downarrow$ 0.27\%)} 
                  & 0.3763 / \textbf{0.3751} {\footnotesize ($\downarrow$ 0.32\%)} 
                  & 0.3759 / \textbf{0.3750} {\footnotesize ($\downarrow$ 0.24\%)} 
                  & 0.3702 / \textbf{0.3639} {\footnotesize ($\downarrow$ 1.70\%)} 
                  & 0.3705 / \textbf{0.3692} {\footnotesize ($\downarrow$ 0.35\%)} \\
            \midrule

            \multirow{3}{*}{\textbf{Database}} 
            & AUC & 0.7467 / \textbf{0.7493} {\footnotesize ($\uparrow$ 0.35\%)}
                  & 0.8064 / \textbf{0.8079} {\footnotesize ($\uparrow$ 0.19\%)}
                  & 0.8060 / \textbf{0.8067} {\footnotesize ($\uparrow$ 0.09\%)}
                  & 0.8067 / \textbf{0.8071} {\footnotesize ($\uparrow$ 0.05\%)}
                  & 0.8359 / \textbf{0.8422} {\footnotesize ($\uparrow$ 0.75\%)}
                  & 0.8133 / \textbf{0.8156} {\footnotesize ($\uparrow$ 0.28\%)}\\
            & ACC & \textbf{0.8294} / 0.8285 {\footnotesize ($\downarrow$ 0.11\%)} 
                  & 0.8383 / \textbf{0.8432} {\footnotesize ($\uparrow$ 0.58\%)}
                  & 0.8395 / \textbf{0.8435} {\footnotesize ($\uparrow$ 0.48\%)}
                  & 0.8392 / \textbf{0.8438} {\footnotesize ($\uparrow$ 0.55\%)}
                  & 0.8478 / \textbf{0.8498} {\footnotesize ($\uparrow$ 0.24\%)}
                  & 0.8460 / \textbf{0.8483} {\footnotesize ($\uparrow$ 0.27\%)} \\
            & RMSE & 0.3562 / \textbf{0.3561} {\footnotesize ($\downarrow$ 0.03\%)}
                  & 0.3415 / \textbf{0.3380} {\footnotesize ($\downarrow$ 1.02\%)} 
                  & 0.3408 / \textbf{0.3382} {\footnotesize ($\downarrow$ 0.76\%)}
                  & 0.3410 / \textbf{0.3381} {\footnotesize ($\downarrow$ 0.85\%)}
                  & 0.3302 / \textbf{0.3274} {\footnotesize ($\downarrow$ 0.85\%)}
                  & 0.3360 / \textbf{0.3343} {\footnotesize ($\downarrow$ 0.51\%)}\\
            \midrule

            \multirow{3}{*}{\textbf{Comp}} 
            & AUC & 0.7073 / \textbf{0.7157} {\footnotesize ($\uparrow$ 1.19\%)} 
                  & 0.7912 / \textbf{0.7958} {\footnotesize ($\uparrow$ 0.58\%)} 
                  & 0.7929 / \textbf{0.7958} {\footnotesize ($\uparrow$ 0.37\%)} 
                  & 0.7914 / \textbf{0.7960} {\footnotesize ($\uparrow$ 0.58\%)} 
                  & 0.8142 / \textbf{0.8263} {\footnotesize ($\uparrow$ 1.49\%)} 
                  & 0.7986 / \textbf{0.8037} {\footnotesize ($\uparrow$ 0.64\%)} \\
            & ACC & 0.8056 / \textbf{0.8077} {\footnotesize ($\uparrow$ 0.26\%)} 
                  & 0.8201 / \textbf{0.8232} {\footnotesize ($\uparrow$ 0.38\%)} 
                  & 0.8195 / \textbf{0.8243} {\footnotesize ($\uparrow$ 0.59\%)} 
                  & 0.8207 / \textbf{0.8245} {\footnotesize ($\uparrow$ 0.46\%)} 
                  & 0.8266 / \textbf{0.8314} {\footnotesize ($\uparrow$ 0.58\%)} 
                  & 0.8111 / \textbf{0.8298} {\footnotesize ($\uparrow$ 2.31\%)} \\
            & RMSE & 0.3794 / \textbf{0.3773} {\footnotesize ($\downarrow$ 0.55\%)} 
                  & 0.3591 / \textbf{0.3558} {\footnotesize ($\downarrow$ 0.92\%)} 
                  & 0.3585 / \textbf{0.3558} {\footnotesize ($\downarrow$ 0.75\%)} 
                  & 0.3587 / \textbf{0.3552} {\footnotesize ($\downarrow$ 0.98\%)} 
                  & 0.3520 / \textbf{0.3454} {\footnotesize ($\downarrow$ 1.87\%)} 
                  & 0.3656 / \textbf{0.3529} {\footnotesize ($\downarrow$ 3.47\%)} \\
            \midrule

            \multirow{3}{*}{\textbf{EdNet}} 
            & AUC & 0.6602 / \textbf{0.6678} {\footnotesize ($\uparrow$ 1.15\%)} 
                  & 0.6868 / \textbf{0.6936} {\footnotesize ($\uparrow$ 0.99\%)} 
                  & 0.6924 / \textbf{0.6977} {\footnotesize ($\uparrow$ 0.77\%)} 
                  & 0.6885 / \textbf{0.6934} {\footnotesize ($\uparrow$ 0.71\%)} 
                  & 0.7282 / \textbf{0.7365} {\footnotesize ($\uparrow$ 1.14\%)} 
                  & \textbf{0.7273} / 0.7239 {\footnotesize ($\downarrow$ 0.47\%)} \\
            & ACC & 0.6189 / \textbf{0.6270} {\footnotesize ($\uparrow$ 1.31\%)} 
                  & 0.6288 / \textbf{0.6351} {\footnotesize ($\uparrow$ 1.00\%)} 
                  & 0.6324 / \textbf{0.6451} {\footnotesize ($\uparrow$ 2.01\%)} 
                  & 0.6337 / \textbf{0.6402} {\footnotesize ($\uparrow$ 1.03\%)} 
                  & 0.6698 / \textbf{0.6800} {\footnotesize ($\uparrow$ 1.52\%)} 
                  & 0.6614 / \textbf{0.6698} {\footnotesize ($\uparrow$ 1.27\%)} \\
            & RMSE & 0.4792 / \textbf{0.4769} {\footnotesize ($\downarrow$ 0.48\%)} 
                  & 0.4848 / \textbf{0.4822} {\footnotesize ($\downarrow$ 0.54\%)} 
                  & 0.4872 / \textbf{0.4864} {\footnotesize ($\downarrow$ 0.16\%)} 
                  & 0.4813 / \textbf{0.4783} {\footnotesize ($\downarrow$ 0.62\%)} 
                  & 0.4744 / \textbf{0.4672} {\footnotesize ($\downarrow$ 1.52\%)} 
                  & 0.4846 / \textbf{0.4786} {\footnotesize ($\downarrow$ 1.24\%)} \\

            \bottomrule
        \end{tabular}
    }
    \caption{
        Comprehensive performance comparison. 
        Rows denote datasets; columns denote models.
        Format: ``\textit{Original} / \textbf{\textit{+TSDR}} {\footnotesize(\%Improv.)}''.
    }
    \label{tab:results_transposed}
\end{table*}

\paragraph{Implementation Details.}We adopt the pipeline from the work in ~\cite{diskt_bias} using student-stratified 5-fold cross-validation. A 10\% validation set determines early stopping after 15 epochs of non-improvement. We truncate sequences to the latest 50 interactions and train models using Adam with a learning rate of 0.001, a batch size of 64, and an embedding dimension of 64. The random seed and dropout rate are set to 42 and 0.05, respectively. Furthermore, we tune the hyperparameter $\lambda$ that controls loss strength within [0.1, 0.3, 0.5, 0.7, 1, 2]. All experiments run on an NVIDIA RTX 4090 GPU and are evaluated via AUC, ACC, and RMSE following ~\cite{lpkt,folibikt_bias,diskt_bias}. The source code is available at \url{https://github.com/plzhan/TSDR}.

\subsection{Results}
\subsubsection{Empirical Study (RQ1)}
Table~\ref{tab:results_transposed} summarizes the performance of TSDR applied to six baseline models across nine datasets. Comparing the original performance with the TSDR-enhanced results, we have the following observations: 1) Regarding the primary evaluation metric of AUC, TSDR yields improvements across diverse model architectures. These gains are particularly notable on datasets characterized by high sparsity or strong student autonomy. For instance, TSDR improves AKT by 5.06\% on the Prob dataset and SparseKT by 4.76\% on Assist17. These results suggest that our framework helps recover mastery patterns by mitigating bias in MNAR data. 2) The magnitude of improvement appears to correlate with the data collection mechanism. While TSDR provides substantial gains on open-ended platforms, the improvements on EdNet and Linux are relatively limited, as exemplified by an AUC gain of approximately 1\% on EdNet. This behavior is likely attributable to EdNet's specific design, which enforces a mandatory ``bundle'' completion policy~\cite{ednet}. Such structural constraints tend to limit selection bias, rendering the missing data mechanism closer to random compared to other datasets, thereby reducing the potential scope for bias correction.
3) Regarding secondary metrics such as ACC and RMSE, TSDR consistently improves calibration. Even on datasets where AUC gains plateau, such as Linux and Database, TSDR frequently reduces RMSE and improves ACC. A specific example is the 1.70\% reduction for SparseKT on Linux. This indicates that counterfactual imputation may contribute to better-calibrated probability estimates, even when the ranking performance remains stable.

\subsubsection{Experiments on Synthetic Data (RQ2)}
To evaluate robustness, we tested TSDR on synthetic datasets with MNAR degrees ($\gamma$) ranging from $0.0$ to $0.999$. 
As shown in Figure~\ref{fig:mnar_trends}, we observe varying degrees of performance degradation across all baselines as the bias severity increases, particularly visible in AUC. 
However, equipping SOTA models with TSDR consistently yields improvements across the entire spectrum. 
Specifically for RMSE, with the exception of DKT, TSDR is highly effective at reducing error when MNAR is low, although this advantage diminishes slightly beyond $\gamma > 0.8$. 
Since TSDR is a model-agnostic strategy, it relies on the backbone's modeling capacity to maximize efficacy: we observe the most significant relative lift in AKT, while DisKT achieves the best overall performance, likely attributed to its disentangled modeling and causal subtraction mechanism~\cite{diskt_bias}.
Furthermore, the positive results at $\gamma=0.0$ indicate that our framework is ``safe'', avoiding performance degradation even when selection bias is minimal, likely due to the regularization effect of the counterfactual imputation (see Appendix~\ref{app:syn_test} for numerical results).
\begin{figure}[ht] 
    \centering
    \includegraphics[width=\columnwidth]{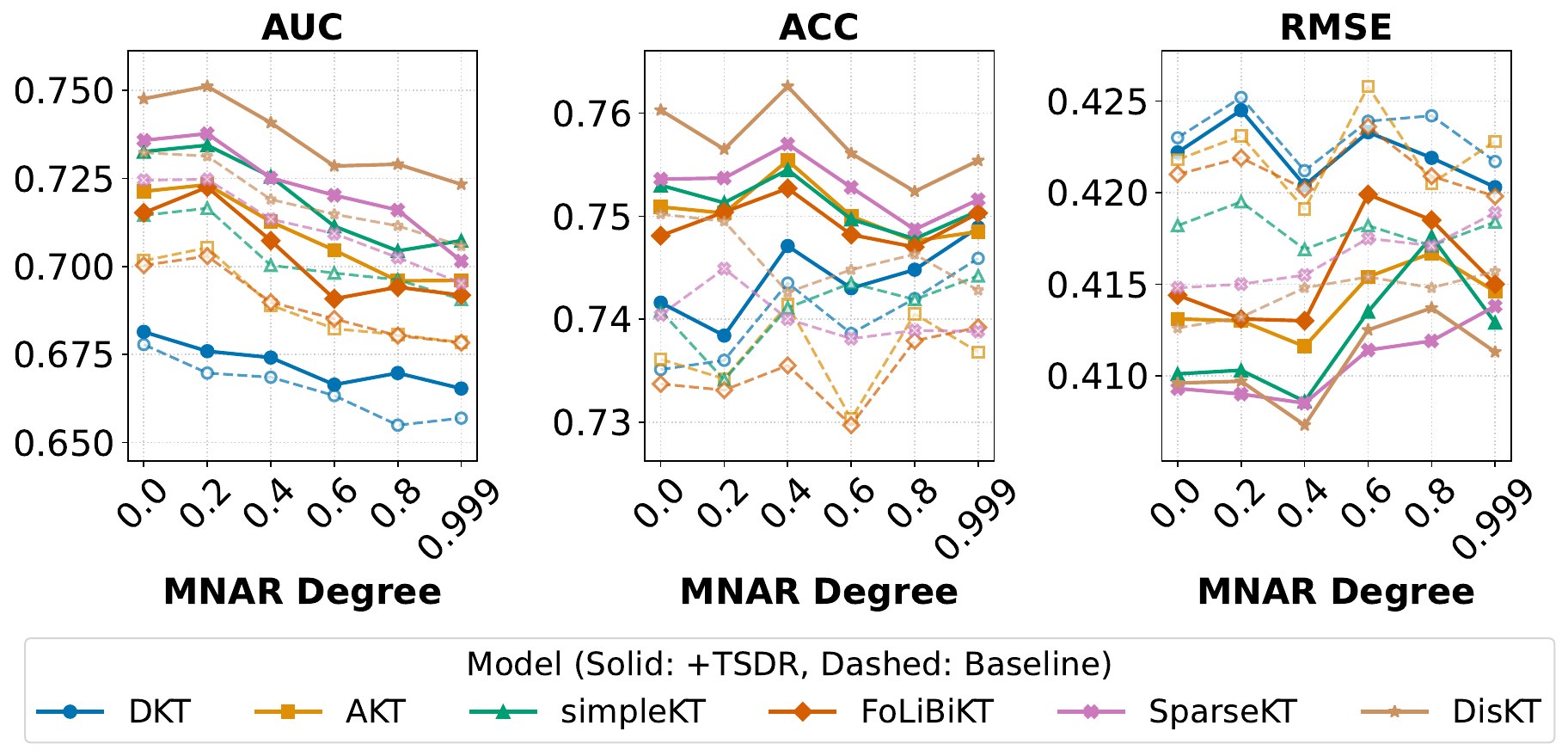}
    \caption{
        Performance of different MNAR levels on synthetic data.
    }
    \label{fig:mnar_trends}
\end{figure}
\subsubsection{Sensitivity Analysis (RQ3)}
We focus on the prob dataset, where our model achieves superior performance, to highlight the experimental results.
\begin{figure}[ht]
    \centering
    \includegraphics[width=\linewidth]{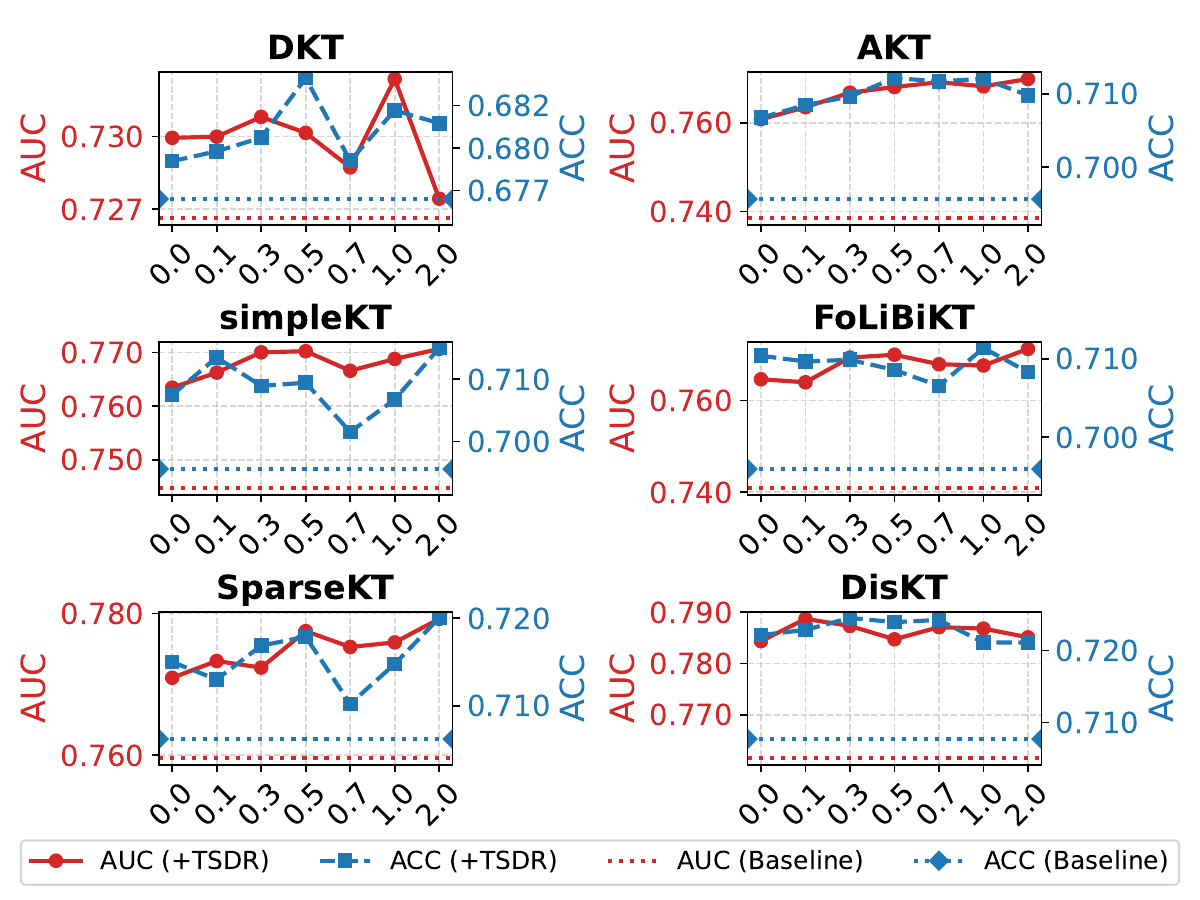}
    \caption{The AUC and ACC of TSDR on varying number of 
$\myparam$.}
    \label{fig:sensitivity}
\end{figure}
Figure~\ref{fig:sensitivity} shows the impact of the hyperparameter $\lambda$ on TSDR. 
We observe that incorporating the temporal smoothness constraint yields superior performance compared to the vanilla DR ($\lambda=0$) within an appropriate range, while outperforming the original baselines.
The method exhibits strong robustness, achieving a stable ``sweet spot'' generally in $\lambda \in [0.3, 1.0]$. 
Despite minor fluctuations in the case, e.g., DKT at $\myparam=0.7$, the overall performance remains high, indicating that TSDR is effective without requiring meticulous tuning.
\section{Conclusion}
This paper introduces TSDR, a framework that addresses selection bias in Knowledge Tracing by integrating a doubly robust estimator with a theoretically-guided temporal smoothness constraint. Experiments across multiple benchmarks demonstrate its superior accuracy and stability over state-of-the-art methods. We acknowledge that the joint optimization of auxiliary modules incurs additional training overhead; however, this cost is confined to the offline phase and does not compromise online inference efficiency. Ultimately, TSDR offers a principled, causally-grounded paradigm for learning from observational data, paving the way for more reliable intelligent tutoring systems.

\section{Acknowledgments}
This research was supported in part by National Science and Technology Major Project (2021ZD0111501), Natural Science Foundation of China (U24A20233, 62206064, 62206061, 62476163, 62406078, 62406080), the Guangdong Basic and Applied Basic Research Foundation (2025A1515010172, 2023B1515120020), the Guangzhou Basic and Applied Basic Research Foundation (2024A04J4384), and CCF-DiDi GAIA Collaborative Research Funds (CCF-DiDi GAIA 202521).
\bibliographystyle{named}
\bibliography{ijcai26}
\clearpage
\appendix
\begin{appendices}
The supplementary material is organized as follows: Appendix~\ref{app:benchmark} presents the sources of the public datasets and their statistical information; Appendix ~\ref{app:syn_test} describes the details of how our simulated data were generated, including the statistical properties of the data and tables showing the performance of different models on them; 
Appendix~\ref{app:baselines} introduces the baseline models we employed; 
Appendix~\ref{app:ablation} further supplements the ablation study and computational overhead analysis in response to the reviewers' comments. It summarizes the contribution of each component in the proposed DR module and reports the additional training cost introduced by DR;
Appendix~\ref{sec:theoretical_main} outlines our main theoretical results to facilitate understanding of the proposed method; and Appendix~\ref{sec:detailed_proofs} provides the proofs corresponding to the theoretical results in Appendix~\ref{sec:theoretical_main}.
\section{Details of Benchmark Dataset}
\label{app:benchmark}
We evaluate the performance of TSDR on 9 public datasets:
\begin{itemize}
    \item \textbf{Spanish}~\cite{spanish}: The spanish dataset is from middle-school students practicing Spanish exercises, including translations and applications of basic skills like verb conjugation, over a 15-week semester.

    \item \textbf{ASSISTments17}~\cite{assistments17}: The ASSISTments17 dataset is derived from the 2017 Longitudinal Data Mining Competition, containing longitudinal student interaction data collected from the ASSISTments online tutoring system.

    \item \textbf{Slepemapy}~\cite{slepemapy}: The slepemapy dataset originates from slepemapy.cz, an online platform dedicated to the adaptive practice of geography facts.

    \item \textbf{Algebra05}~\cite{algebra}: The algebra05 dataset comes from the KDD Cup 2010 EDM Challenge, containing detailed step-level student responses to algebra questions from the Carnegie Learning system.

    \item \textbf{Prob, Linux, Database, Comp (PTADisc)}~\cite{ptadisc}: The prob, comp, linux, database datasets are collected from the Programming Teaching Assistant platform, specifically from course exercises in Probability and Statistics, Computational Thinking, Linux System, and Database Technology and Application.

    \item \textbf{EdNet}~\cite{ednet}: The ednet dataset, collected by the multi-platform AI tutoring service Santa, stands as the largest publicly released interactive educational system dataset to date.
\end{itemize}

\begin{table}[ht]
    \centering
    \resizebox{\columnwidth}{!}{ 
        \begin{tabular}{ccccc} 
            \toprule
            \textbf{Dataset} & \textbf{\# Users} & \textbf{\# Questions} & \textbf{\# concepts} & \textbf{\# Interactions} \\ 
            \midrule
            Spanish       & 182 & 409 & 221 & 578,726 \\
            ASSIST17      & 1,708 & 3,162 & 411 & 934,638 \\
            Slepemapy     & 5,000 & 2,723 & 1,391 & 625,523 \\
            Algebra05     & 571 & 173,113 & 271 & 607,014 \\
            Prob          & 512  & 1,054  & 247  & 42,869 \\
            Linux         & 4,375  & 2,672  & 281  & 365,027 \\
            Database      & 5,488  & 3,388  & 291  & 990,468 \\
            Comp          & 5,000  & 7,460  & 445  & 668,927 \\
            EdNet         & 5,000  & 12,002  & 1,769  & 789,812 \\
            \bottomrule
        \end{tabular}
    }
    \caption{Statistics of the 9 datasets}
    \label{tab:stats}
\end{table}
Table~\ref{tab:stats} presents the statistics of the processed datasets, excluding student sequences with fewer than 5 interactions.

\section{Details of the Simulation Study}
\label{app:syn_test}
This section details the data generation process for our simulation study. The base dataset was generated for 1,000 students and 200 questions, with the questions covering 20 distinct knowledge components. The experimental results on this synthetic dataset are presented in Table~\ref{tab:synthetic_results}. The process incorporates four key characteristics:

\begin{enumerate}
    \item \textbf{Long-tail Distribution}: To mimic the real-world scenario where some knowledge points are practiced far more often than others, the frequency of items (questions) was modeled to follow a Zipfian distribution with a shape parameter of $\alpha=0.8$.

    \item \textbf{Dynamic Learning}: A student's mastery of a knowledge component is not static. We simulated dynamic learning where a student's mastery level evolves over time. Upon answering a question correctly, their mastery of the associated knowledge components increases with a learning rate of $0.55$.

    \item \textbf{Response Noise}: Student responses are not always a perfect reflection of their true mastery. To account for this uncertainty, we introduced a guessing probability of $0.1$ (the probability of a non-proficient student answering correctly) and a slipping probability of $0.05$ (the probability of a proficient student answering incorrectly).

    \item \textbf{Controllable MNAR Mechanism}: This is the core mechanism for simulating selection bias. Interactions are selectively marked as "skipped" (missing) based on the ground-truth probability of success, denoted as $p$. This simulates students strategically avoiding questions they perceive as too difficult or too easy. The skipping probability is determined by the MNAR degree $\gamma$ and the extremeness of $p$:
    \begin{itemize}
        \item Too difficult ($p < 0.25$): The skipping probability is $\gamma \cdot (1-p)$, meaning students are more likely to skip harder questions (lower $p$).
        \item Too trivial ($p > 0.75$): The skipping probability is $\gamma \cdot p$, meaning students are more likely to skip easier questions (higher $p$).
    \end{itemize}
    To create a spectrum of bias scenarios for our robustness evaluation, we generated six distinct datasets by setting the MNAR degree $\gamma$ to values across the set $\{0, 0.2, 0.4, 0.6, 0.8, 0.999\}$. A $\gamma$ of 0 corresponds to a dataset with no selection bias, while a $\gamma$ of 0.999 represents a very strong selection bias. Dataset statistics are summarized in Table~\ref{tab:synthetic_stats}.
\end{enumerate}

\begin{table}[ht]
    \centering
    \resizebox{\columnwidth}{!}{ 
        \begin{tabular}{ccccc}
            \toprule
            \textbf{Dataset} & \textbf{Degree ($\gamma$)} & \textbf{\# Observed} & \textbf{\# Skipped} & \textbf{\# Interactions} \\
            \midrule
            Synthetic-1 & 0.0 & 114482 & 0 & 48,967 \\ 
            Synthetic-2 & 0.2 & 116791 & 12922 & 48,476 \\ 
            Synthetic-3 & 0.4 & 114214 & 24669 & 48,578 \\ 
            Synthetic-4 & 0.6 & 114387 & 36055 & 48,091 \\ 
            Synthetic-5 & 0.8 & 114789 & 46627 & 47,826 \\ 
            Synthetic-6 & 0.999 & 114191 & 55770 & 47,645 \\ 
            \bottomrule
        \end{tabular}
    }
    \caption{Statistics of the generated synthetic datasets under varying degrees of Missing Not At Random (MNAR).}
    \label{tab:synthetic_stats}
\end{table}

\begin{table*}[ht]
    \centering
    \resizebox{\textwidth}{!}{ 
        \begin{tabular}{llcccccc}
            \toprule
            \textbf{Dataset} & \textbf{Metric} & \textbf{DKT} & \textbf{AKT}& \textbf{simpleKT}& \textbf{FoLiBiKT}& \textbf{SparseKT}& \textbf{DisKT}\\
            \midrule
            
            \multirow{3}{*}{\shortstack{\textbf{Synthetic}\\(MNAR 0.999)}} 
           & AUC $\uparrow$  & 0.6569 / \textbf{0.6653} {\footnotesize ($\uparrow$ 1.28\%)} 
                   & 0.6784 / \textbf{0.6960} {\footnotesize ($\uparrow$ 2.59\%)} 
                   & 0.6906 / \textbf{0.7073} {\footnotesize ($\uparrow$ 2.42\%)} 
                   & 0.6783 / \textbf{0.6918} {\footnotesize ($\uparrow$ 1.99\%)} 
                   & 0.6952 / \textbf{0.7015} {\footnotesize ($\uparrow$ 0.91\%)} 
                   & 0.7058 / \textbf{0.7233} {\footnotesize ($\uparrow$ 2.48\%)} \\
            & ACC $\uparrow$  & 0.7459 / \textbf{0.7489} {\footnotesize ($\uparrow$ 0.40\%)} 
                   & 0.7368 / \textbf{0.7485} {\footnotesize ($\uparrow$ 1.59\%)} 
                   & 0.7442 / \textbf{0.7505} {\footnotesize ($\uparrow$ 0.85\%)} 
                   & 0.7392 / \textbf{0.7503} {\footnotesize ($\uparrow$ 1.50\%)} 
                   & 0.7388 / \textbf{0.7516} {\footnotesize ($\uparrow$ 1.73\%)} 
                   & 0.7428 / \textbf{0.7554} {\footnotesize ($\uparrow$ 1.70\%)} \\
            & RMSE $\downarrow$ & 0.4217 / \textbf{0.4203} {\footnotesize ($\downarrow$ 0.33\%)} 
                   & 0.4228 / \textbf{0.4146} {\footnotesize ($\downarrow$ 1.94\%)} 
                   & 0.4184 / \textbf{0.4129} {\footnotesize ($\downarrow$ 1.31\%)} 
                   & 0.4198 / \textbf{0.4150} {\footnotesize ($\downarrow$ 1.14\%)} 
                   & 0.4189 / \textbf{0.4138} {\footnotesize ($\downarrow$ 1.22\%)} 
                   & 0.4157 / \textbf{0.4113} {\footnotesize ($\downarrow$ 1.06\%)} \\
            \midrule

            \multirow{3}{*}{\shortstack{\textbf{Synthetic}\\(MNAR 0.8)}} 
            & AUC  & 0.6549 / \textbf{0.6697} {\footnotesize ($\uparrow$ 2.26\%)} 
                   & 0.6806 / \textbf{0.6959} {\footnotesize ($\uparrow$ 2.25\%)} 
                   & 0.6962 / \textbf{0.7044} {\footnotesize ($\uparrow$ 1.18\%)} 
                   & 0.6803 / \textbf{0.6941} {\footnotesize ($\uparrow$ 2.03\%)} 
                   & 0.7025 / \textbf{0.7160} {\footnotesize ($\uparrow$ 1.92\%)} 
                   & 0.7115 / \textbf{0.7290} {\footnotesize ($\uparrow$ 2.46\%)} \\
            & ACC  & 0.7420 / \textbf{0.7448} {\footnotesize ($\uparrow$ 0.38\%)} 
                   & 0.7405 / \textbf{0.7476} {\footnotesize ($\uparrow$ 0.96\%)} 
                   & 0.7419 / \textbf{0.7478} {\footnotesize ($\uparrow$ 0.80\%)} 
                   & 0.7379 / \textbf{0.7470} {\footnotesize ($\uparrow$ 1.23\%)} 
                   & 0.7389 / \textbf{0.7487} {\footnotesize ($\uparrow$ 1.33\%)} 
                   & 0.7463 / \textbf{0.7524} {\footnotesize ($\uparrow$ 0.82\%)} \\
            & RMSE & 0.4242 / \textbf{0.4219} {\footnotesize ($\downarrow$ 0.54\%)} 
                   & 0.4205 / \textbf{0.4167} {\footnotesize ($\downarrow$ 0.90\%)} 
                   & \textbf{0.4171} / 0.4176 {\footnotesize ($\uparrow$ 0.12\%)} 
                   & 0.4209 / \textbf{0.4185} {\footnotesize ($\downarrow$ 0.57\%)} 
                   & 0.4171 / \textbf{0.4119} {\footnotesize ($\downarrow$ 1.25\%)} 
                   & 0.4148 / \textbf{0.4137} {\footnotesize ($\downarrow$ 0.27\%)} \\
            \midrule

            \multirow{3}{*}{\shortstack{\textbf{Synthetic}\\(MNAR 0.6)}} 
            & AUC  & 0.6633 / \textbf{0.6664} {\footnotesize ($\uparrow$ 0.47\%)} 
                   & 0.6823 / \textbf{0.7046} {\footnotesize ($\uparrow$ 3.27\%)}
                   & 0.6981 / \textbf{0.7114} {\footnotesize ($\uparrow$ 1.91\%)}
                   & 0.6851 / \textbf{0.6908} {\footnotesize ($\uparrow$ 0.83\%)}
                   & 0.7093 / \textbf{0.7202} {\footnotesize ($\uparrow$ 1.54\%)} 
                   & 0.7148 / \textbf{0.7285} {\footnotesize ($\uparrow$ 1.92\%)} \\
            & ACC  & 0.7386 / \textbf{0.7430} {\footnotesize ($\uparrow$ 0.60\%)} 
                   & 0.7303 / \textbf{0.7500} {\footnotesize ($\uparrow$ 2.70\%)}
                   & 0.7435 / \textbf{0.7497} {\footnotesize ($\uparrow$ 0.83\%)}
                   & 0.7297 / \textbf{0.7482} {\footnotesize ($\uparrow$ 2.54\%)}
                   & 0.7381 / \textbf{0.7528} {\footnotesize ($\uparrow$ 1.99\%)} 
                   & 0.7448 / \textbf{0.7561} {\footnotesize ($\uparrow$ 1.52\%)} \\
            & RMSE & 0.4239 / \textbf{0.4233} {\footnotesize ($\downarrow$ 0.14\%)} 
                   & 0.4258 / \textbf{0.4154} {\footnotesize ($\downarrow$ 2.44\%)}
                   & 0.4182 / \textbf{0.4135} {\footnotesize ($\downarrow$ 1.12\%)} 
                   & 0.4236 / \textbf{0.4199} {\footnotesize ($\downarrow$ 0.87\%)} 
                   & 0.4175 / \textbf{0.4114} {\footnotesize ($\downarrow$ 1.46\%)}
                   & 0.4154 / \textbf{0.4125} {\footnotesize ($\downarrow$ 0.70\%)} \\
            \midrule

            \multirow{3}{*}{\shortstack{\textbf{Synthetic}\\(MNAR 0.4)}} 
            & AUC  & 0.6685 / \textbf{0.6741} {\footnotesize ($\uparrow$ 0.84\%)} 
                   & 0.6890 / \textbf{0.7127} {\footnotesize ($\uparrow$ 3.44\%)} 
                   & 0.7003 / \textbf{0.7255} {\footnotesize ($\uparrow$ 3.60\%)} 
                   & 0.6898 / \textbf{0.7073} {\footnotesize ($\uparrow$ 2.54\%)} 
                   & 0.7134 / \textbf{0.7251} {\footnotesize ($\uparrow$ 1.64\%)} 
                   & 0.7190 / \textbf{0.7408} {\footnotesize ($\uparrow$ 3.03\%)} \\
            & ACC  & 0.7435 / \textbf{0.7471} {\footnotesize ($\uparrow$ 0.48\%)}
                   & 0.7414 / \textbf{0.7554} {\footnotesize ($\uparrow$ 1.89\%)} 
                   & 0.7411 / \textbf{0.7545} {\footnotesize ($\uparrow$ 1.81\%)} 
                   & 0.7355 / \textbf{0.7527} {\footnotesize ($\uparrow$ 2.34\%)}
                   & 0.7400 / \textbf{0.7570} {\footnotesize ($\uparrow$ 2.30\%)} 
                   & 0.7426 / \textbf{0.7626} {\footnotesize ($\uparrow$ 2.69\%)} \\
            & RMSE & 0.4212 / \textbf{0.4204} {\footnotesize ($\downarrow$ 0.19\%)}
                   & 0.4191 / \textbf{0.4116} {\footnotesize ($\downarrow$ 1.79\%)}
                   & 0.4169 / \textbf{0.4086} {\footnotesize ($\downarrow$ 1.99\%)} 
                   & 0.4202 / \textbf{0.4130} {\footnotesize ($\downarrow$ 1.71\%)}
                   & 0.4155 / \textbf{0.4085} {\footnotesize ($\downarrow$ 1.68\%)}
                   & 0.4148 / \textbf{0.4073} {\footnotesize ($\downarrow$ 1.81\%)} \\
            \midrule

            \multirow{3}{*}{\shortstack{\textbf{Synthetic}\\(MNAR 0.2)}} 
            & AUC  & 0.6697 / \textbf{0.6759} {\footnotesize ($\uparrow$ 0.93\%)} 
                   & 0.7053 / \textbf{0.7232} {\footnotesize ($\uparrow$ 2.54\%)}
                   & 0.7165 / \textbf{0.7344} {\footnotesize ($\uparrow$ 2.50\%)}
                   & 0.7029 / \textbf{0.7225} {\footnotesize ($\uparrow$ 2.79\%)}
                   & 0.7247 / \textbf{0.7377} {\footnotesize ($\uparrow$ 1.79\%)}
                   & 0.7313 / \textbf{0.7511} {\footnotesize ($\uparrow$ 2.71\%)} \\
            & ACC  & 0.7360 / \textbf{0.7384} {\footnotesize ($\uparrow$ 0.33\%)}
                   & 0.7342 / \textbf{0.7503} {\footnotesize ($\uparrow$ 2.19\%)}
                   & 0.7341 / \textbf{0.7513} {\footnotesize ($\uparrow$ 2.34\%)}
                   & 0.7331 / \textbf{0.7504} {\footnotesize ($\uparrow$ 2.36\%)}
                   & 0.7449 / \textbf{0.7537} {\footnotesize ($\uparrow$ 1.18\%)}
                   & 0.7495 / \textbf{0.7565} {\footnotesize ($\uparrow$ 0.93\%)} \\
            & RMSE & 0.4252 / \textbf{0.4245} {\footnotesize ($\downarrow$ 0.16\%)}
                   & 0.4231 / \textbf{0.4130} {\footnotesize ($\downarrow$ 2.39\%)}
                   & 0.4195 / \textbf{0.4103} {\footnotesize ($\downarrow$ 2.19\%)}
                   & 0.4219 / \textbf{0.4131} {\footnotesize ($\downarrow$ 2.09\%)}
                   & 0.4150 / \textbf{0.4090} {\footnotesize ($\downarrow$ 1.45\%)}
                   & 0.4132 / \textbf{0.4097} {\footnotesize ($\downarrow$ 0.85\%)} \\
            \midrule

            \multirow{3}{*}{\shortstack{\textbf{Synthetic}\\(MNAR 0.0)}} 
            & AUC  & 0.6778 / \textbf{0.6814} {\footnotesize ($\uparrow$ 0.53\%)}
                   & 0.7016 / \textbf{0.7213} {\footnotesize ($\uparrow$ 2.81\%)}
                   & 0.7145 / \textbf{0.7326} {\footnotesize ($\uparrow$ 2.53\%)}
                   & 0.7003 / \textbf{0.7152} {\footnotesize ($\uparrow$ 2.13\%)}
                   & 0.7245 / \textbf{0.7358} {\footnotesize ($\uparrow$ 1.56\%)} 
                   & 0.7323 / \textbf{0.7476} {\footnotesize ($\uparrow$ 2.09\%)} \\
            & ACC  & 0.7351 / \textbf{0.7416} {\footnotesize ($\uparrow$ 0.88\%)} 
                   & 0.7361 / \textbf{0.7509} {\footnotesize ($\uparrow$ 2.01\%)}
                   & 0.7408 / \textbf{0.7530} {\footnotesize ($\uparrow$ 1.65\%)}
                   & 0.7337 / \textbf{0.7481} {\footnotesize ($\uparrow$ 1.96\%)}
                   & 0.7404 / \textbf{0.7536} {\footnotesize ($\uparrow$ 1.78\%)}
                   & 0.7502 / \textbf{0.7603} {\footnotesize ($\uparrow$ 1.35\%)} \\
            & RMSE & 0.4230 / \textbf{0.4222} {\footnotesize ($\downarrow$ 0.19\%)} 
                   & 0.4218 / \textbf{0.4131} {\footnotesize ($\downarrow$ 2.06\%)}
                   & 0.4182 / \textbf{0.4101} {\footnotesize ($\downarrow$ 1.94\%)}
                   & 0.4210 / \textbf{0.4144} {\footnotesize ($\downarrow$ 1.57\%)}
                   & 0.4148 / \textbf{0.4093} {\footnotesize ($\downarrow$ 1.33\%)}
                   & 0.4126 / \textbf{0.4096} {\footnotesize ($\downarrow$ 0.73\%)} \\
            \bottomrule
        \end{tabular}
    }
    \caption{
        Performance comparison on synthetic datasets with varying degrees of MNAR. 
        Rows represent datasets (MNAR degree), and columns represent models.
        The results are presented in the format of ``\textit{Original} / \textbf{\textit{+TSDR}} \footnotesize{(\%Improv.)}''.
        $\uparrow$ indicates a numerical increase, $\downarrow$ indicates a numerical decrease.
        Bold values highlight the better performance in each pair.
    }
    
    \label{tab:synthetic_results}
\end{table*}

\section{Details of Baseline Models}
\label{app:baselines}

To evaluate the universality and effectiveness of our proposed TSDR framework, we selected 6 representative and state-of-the-art Knowledge Tracing models as backbones. These baselines include models built on Recurrent Neural Networks (RNNs) and attention mechanisms. The details of each baseline are as follows:

\begin{itemize}
    \item \textbf{DKT}~\cite{DKT}: DKT is the pioneering work that introduced deep learning to knowledge tracing. It utilizes LSTM networks to model the student's knowledge state as a hidden vector that evolves over time. DKT captures the sequential dependencies in the interaction history to predict future performance.
    
    \item \textbf{AKT}~\cite{AKT}: AKT adapts the Transformer architecture for KT tasks. It introduces a context-aware monotonic attention mechanism to weigh the importance of past interactions relative to the current question. AKT significantly improves interpretability and performance by explicitly capturing the decay of memory over time using exponential decay kernels.
    
    \item \textbf{simpleKT}~\cite{simplekt}: simpleKT challenges the complexity of previous models by demonstrating that a simple but strong baseline can outperform complex architectures. It focuses on capturing question-specific variations using specific embedding strategies without heavy reliance on complex attention layers, highlighting the importance of input feature engineering.
    
    \item \textbf{FoLiBiKT}~\cite{folibikt_bias}: Addressing the limitation that standard attention mechanisms may not fully capture human memory characteristics, FoLiBiKT incorporates a forgetting-aware linear bias into the attention scores. It explicitly models the forgetting curve based on the time interval and the number of intervening interactions, making the attention mechanism more cognitively plausible.
    
    \item \textbf{SparseKT}~\cite{sparsekt}: Recognizing that standard attention mechanisms often assign non-zero weights to irrelevant history (noise), sparseKT introduces a $k$-sparse attention mechanism. It filters out noisy or irrelevant historical interactions by retaining only the top-$k$ most relevant attention scores, thereby enhancing the model's robustness and focusing on the most critical evidence for mastery tracing.
    
    \item \textbf{DisKT}~\cite{diskt_bias}: disKT aims to solve the confounding issues in student performance prediction. It utilizes a disentangled representation learning framework to separate the student's intrinsic ability from the question's difficulty and other confounding factors. By explicitly modeling these factors separately, disKT achieves better generalization and interpretability, representing the state-of-the-art in handling cognitive biases in KT.
\end{itemize}


\begin{table}[ht]
    \centering
    \caption{Ablation study on the effects of joint learning and the placement of temporal smoothness (TS). ``TS on Imp.'' denotes temporal smoothness applied to the imputation branch; ``TS on KT'' denotes temporal smoothness applied to the KT backbone.}
    \label{tab:ablation_ts_joint}
    \resizebox{\linewidth}{!}{
    \begin{tabular}{cccccccc}
    \toprule
    Backbone & DR & Joint Learn. & TS on Imp. & TS on KT & AUC $\uparrow$ & ACC $\uparrow$ & RMSE $\downarrow$ \\
    \midrule
    
    \multirow{4}{*}{DKT}
    & $\checkmark$ &  &  &  & 0.7258 & 0.6796 & 0.4533 \\
    & $\checkmark$ & $\checkmark$ &  &  & 0.7300 & 0.6792 & 0.4520 \\
    & $\checkmark$ & $\checkmark$ &  & $\checkmark$ & 0.7267 & 0.6753 & 0.4545 \\
    & $\checkmark$ & $\checkmark$ & $\checkmark$ &  & \textbf{0.7319} & \textbf{0.6822} & \textbf{0.4510} \\
    \midrule
    
    \multirow{4}{*}{AKT}
    & $\checkmark$ &  &  &  & 0.7420 & 0.6932 & 0.4576 \\
    & $\checkmark$ & $\checkmark$ &  &  & 0.7608 & 0.7067 & 0.4443 \\
    & $\checkmark$ & $\checkmark$ &  & $\checkmark$ & 0.7596 & 0.7055 & 0.4467 \\
    & $\checkmark$ & $\checkmark$ & $\checkmark$ &  & \textbf{0.7719} & \textbf{0.7136} & \textbf{0.4407} \\
    \midrule
    
    \multirow{4}{*}{SimpleKT}
    & $\checkmark$ &  &  &  & 0.7466 & 0.6840 & 0.4560 \\
    & $\checkmark$ & $\checkmark$ &  &  & 0.7635 & 0.7075 & \textbf{0.4465} \\
    & $\checkmark$ & $\checkmark$ &  & $\checkmark$ & 0.7599 & 0.7078 & 0.4504 \\
    & $\checkmark$ & $\checkmark$ & $\checkmark$ &  & \textbf{0.7716} & \textbf{0.7110} & 0.4466 \\
    \midrule
    
    \multirow{4}{*}{FoLiBiKT}
    & $\checkmark$ &  &  &  & 0.7415 & 0.6870 & 0.4638 \\
    & $\checkmark$ & $\checkmark$ &  &  & 0.7646 & 0.7104 & 0.4406 \\
    & $\checkmark$ & $\checkmark$ &  & $\checkmark$ & 0.7608 & 0.7083 & 0.4447 \\
    & $\checkmark$ & $\checkmark$ & $\checkmark$ &  & \textbf{0.7703} & \textbf{0.7126} & \textbf{0.4403} \\  
    \midrule
    
    \multirow{4}{*}{SparseKT}
    & $\checkmark$ &  &  &  & 0.7639 & 0.7047 & 0.4477 \\
    & $\checkmark$ & $\checkmark$ &  &  & 0.7709 & 0.7150 & 0.4499 \\
    & $\checkmark$ & $\checkmark$ &  & $\checkmark$ & 0.7656 & 0.7090 & 0.4456 \\
    & $\checkmark$ & $\checkmark$ & $\checkmark$ &  & \textbf{0.7775} & \textbf{0.7197} & \textbf{0.4441} \\
    \midrule
    
    \multirow{4}{*}{DisKT}
    & $\checkmark$ &  &  &  & 0.7669 & 0.7121 & 0.4441 \\
    & $\checkmark$ & $\checkmark$ &  &  & 0.7844 & 0.7222 & 0.4365 \\
    & $\checkmark$ & $\checkmark$ &  & $\checkmark$ & 0.7818 & 0.7212 & 0.4406 \\
    & $\checkmark$ & $\checkmark$ & $\checkmark$ &  & \textbf{0.7870} & \textbf{0.7262} & \textbf{0.4364} \\
    \bottomrule
    \end{tabular}
    }
\end{table}

\begin{table}[ht]
    \centering
    \captionof{table}{Wall-Clock Time and Memory Overhead.}
    \label{tab:pairwise_delta}
    \setlength{\tabcolsep}{3pt}
    \resizebox{\linewidth}{!}{
    \begin{tabular}{l|ccc|ccc}
    \toprule
    Backbone & Base t/Epoch (s) & +DR t/Epoch (s) &  $\Delta$t (s) & Base Mem. (MB) & +DR Mem. (MB) &  $\Delta$Mem. (MB) \\
    \midrule
    DKT      & 0.10 & 0.31 & +0.21  & 1841.2 & 1901.2  & +60.0 \\
    AKT      & 0.43 & 2.04 & +1.61  & 2037.2 & 8945.2  & +6908.0 \\
    SimpleKT & 0.23 & 1.01 & +0.78  & 1815.2 & 9365.2  & +7550.0 \\
    FoliBiKT & 0.43 & 2.11 & +1.68  & 2235.2 & 9383.2  & +7148.0 \\
    SparseKT & 0.25 & 1.22 & +0.97  & 1837.2 & 10527.2 & +8690.0 \\
    DisKT    & 0.37 & 1.60 & +1.23  & 1837.2 & 9119.2  & +7282.0 \\
    \bottomrule
    \end{tabular}
    }
\end{table}

\section{Details of the Ablation Study and Computational Overhead}
\label{app:ablation}
 We further conduct an ablation study and report the wall-clock time and memory overhead of the proposed DR module. Specifically, we investigate two key questions: (1) whether joint learning brings consistent benefits, and (2) where temporal smoothness should be imposed to better support knowledge tracing. The results are summarized in Tables~\ref{tab:ablation_ts_joint} and~\ref{tab:pairwise_delta}.

Table~\ref{tab:ablation_ts_joint} presents the ablation results across six representative KT backbones. Compared with using DR alone, introducing joint learning generally improves the prediction performance, indicating that the imputation model provides useful auxiliary supervision for the KT backbone. Moreover, applying temporal smoothness to the imputation branch consistently achieves the best or near-best results across different backbones. In contrast, directly imposing temporal smoothness on the KT backbone often leads to inferior performance. This suggests that temporal smoothness is more suitable for regularizing the auxiliary imputation process, where it encourages stable and coherent latent hints, rather than directly constraining the KT backbone.

Overall, the full variant, i.e., DR with joint learning and temporal smoothness on the imputation branch, obtains the best AUC and ACC on all six backbones. These results verify that the proposed design is backbone-agnostic and can be integrated with both recurrent and attention-based KT models.

Table~\ref{tab:pairwise_delta} reports the time and memory overhead introduced by DR. The additional cost is moderate in terms of wall-clock time: the per-epoch training time increases by 0.21s for DKT and by 0.78--1.68s for attention-based backbones. The memory overhead is more noticeable for attention-based models, mainly because the imputation branch introduces additional sequence-level representations and joint optimization requires keeping extra intermediate activations during training. Nevertheless, the additional time and memory overhead is incurred only during training, since the DR phase is used exclusively to guide the KT backbone during optimization. During inference, the DR phase is not involved, and prediction is performed using only the KT backbone without introducing any extra computational overhead. These results indicate that DR improves the prediction accuracy of knowledge tracing with additional cost limited to the training stage.

\section{Theoretical Analysis and Main Results}
\label{sec:theoretical_main}

In this section, we establish the theoretical foundation of our framework. We first quantify the generalization error of the Doubly Robust estimator in the sequential setting, decomposing it into bias and variance components. Crucially, we identify that the estimation risk is dominated by the variance of imputation errors. To address this, we provide the theoretical justification for the \textbf{Temporal Smoothness} constraint, demonstrating that it effectively tightens the generalization bound by restricting the geometric complexity of the latent trajectory. Detailed proofs for all theorems are provided in Appendix \ref{sec:detailed_proofs}.

\subsection{Problem Setup and Generalization Bound}
Let $T$ be the sequence length and $\mathcal{C}$ be the set of skills. The dataset $\mathcal{D}_\text{full}$ consists of pairs $(t, c)$ with size $N$. Let $e_{t+1,c}$ be the true prediction error, $\hat{e}_{t+1,c}$ be the imputed error, and $\hat{p}_{t+1,c}$ be the estimated propensity.
We define the \textbf{True Risk} $\mathcal{R}_\text{true} = \frac{1}{N} \sum e_{t+1,c}$ and the \textbf{DR Estimator} $\hat{\mathcal{R}}_\text{DR} = \frac{1}{N} \sum (\hat{e}_{t+1,c} + \frac{o_{t+1,c}(e_{t+1,c} - \hat{e}_{t+1,c})}{\hat{p}_{t+1,c}})$.

Our first main result establishes the generalization bound for the DR estimator in this context:

\begin{theorem}[Generalization Bound of DR Estimator]
Assuming bounded errors, for any finite hypothesis space $\mathcal{H}$, with probability at least $1-\eta$, the true risk is bounded by:
\begin{equation}
\label{eq:main_bound_result}
\begin{split}
    \mathcal{R}_{true} \leq \hat{\mathcal{R}}_\text{DR} &+ \underbrace{\frac{1}{N}\sum_{(t,c)\in\mathcal{D}_\text{full}} |\Delta_{t+1,c} \cdot\delta_{t+1,c}|}_{\text{Bias Term}} \\&+ \underbrace{\sqrt{\frac{\ln(2|\mathcal{H}|/\eta)}{2N^2} \sum_{(t,c)\in\mathcal{D}_\text{full}} \left( \frac{\delta_{t+1,c}}{\hat{p}_{t+1,c}} \right)^2}}_{\text{Variance Term}},
\end{split}
\end{equation}
where $\delta_{t,c} = e_{t,c} - \hat{e}_{t,c}$ is the imputation error deviation, and $\Delta_{t,c}$ is the relative propensity error.
\end{theorem}

\paragraph{Discussion.} Theorem \ref{thm:generalization_bound} reveals the trade-off in the DR estimator. The \textbf{Bias Term} benefits from the double robustness property: it remains small if either the propensity model ($\Delta \approx 0$) or the imputation model ($\delta \approx 0$) is accurate. However, the critical bottleneck is the \textbf{Variance Term}, specifically the summation $\sum (\delta_{t,c} / \hat{p}_{t,c})^2$. Since this term scales with the squared imputation error, instability in the imputation model $g_\phi$ can loosely bound the true risk, leading to poor generalization. Therefore, to stabilize training and tighten this bound, we propose joint training of the prediction and imputation models to simultaneously minimize both prediction and imputation inaccuracies. In particular, reducing the complexity of the imputation model is imperative, as it effectively controls and decreases $  \delta_{t,c}  $.

\subsection{Theoretical Justification of Temporal Smoothness}
\subsubsection{Generalization Bound via Rademacher Complexity}

Let \(\mathcal{G}\) be the class of imputation functions parameterized by \(\phi\). We consider the learning of \(g_\phi\) on the sequence of latent states \(H = \{h_1, h_2, \dots, h_T\}\). Assume \(h_t \in \mathbb{R}^d\), and functions in \(\mathcal{G}\) satisfy \(L_g\)-Lipschitz continuity and boundedness \(|g(h)| \leq M\).
\begin{theorem}[Generalization Bound for Imputation Model, adapted from \cite{mohri2018foundations}]
\label{thm:rademacher_bound}
For any \(\eta > 0\), with probability at least \(1-\eta\), the generalization error of the imputation model is bounded by:
\begin{equation}
    \mathbb{E}[\delta^2] \leq \hat{\mathbb{E}}[\delta^2] + 2\mathfrak{R}_T(\mathcal{G} \circ H) + \sqrt{\frac{\log(1/\eta)}{2T}},
\end{equation}
where \(\mathfrak{R}_T(\mathcal{G} \circ H)\) denotes the empirical Rademacher Complexity of the function class \(\mathcal{G}\) acting on the input trajectory \(H\).
\end{theorem}

Theorem \ref{thm:rademacher_bound} indicates that minimizing the Rademacher complexity \(\mathfrak{R}_T\) is essential for reducing the true error variance. 

To minimize the critical variance term identified above, we employ Statistical Learning Theory to link the \textbf{Temporal Smoothness} loss, $\mathcal{L}_{smooth} = \sum \|h_t - h_{t-1}\|_2^2$, to the Rademacher complexity of the model.
\begin{theorem}[Smoothness-Regularized Complexity Bound]
\label{thm:smoothness_main}
Assume the imputation function class $\mathcal{G}$ is $L_g$-Lipschitz continuous. Minimizing the Temporal Smoothness loss $\mathcal{L}_{smooth}$ effectively reduces the Empirical Rademacher Complexity $\mathfrak{R}_T$ of the imputation model. Specifically:
\begin{equation}
    \mathfrak{R}_T(\mathcal{G} \circ H) \leq \frac{12}{\sqrt{T}} \int_0^D \sqrt{\mathcal{C}_\text{smooth}(\epsilon) \cdot \log \left( \frac{2M}{\epsilon} + 1 \right)} d\epsilon,
\end{equation}
where the complexity term $\mathcal{C}_{smooth}(\epsilon) = \frac{L_g \sqrt{T \cdot \mathcal{L}_{smooth}}}{\epsilon} + 1$ is monotonically increasing with respect to $\mathcal{L}_{smooth}$.
\end{theorem}

\paragraph{Discussion.} Theorem \ref{thm:smoothness_main} provides the theoretical backbone for our proposed regularizer. It establishes a causal chain:
\begin{enumerate}
    \item Minimizing $\mathcal{L}_{smooth}$ restricts the total path length of the latent trajectory $H$ in the high-dimensional space.
    \item A shorter path length reduces the \textit{Covering Number} of the trajectory (geometrically simpler).
    \item A reduced Covering Number lowers the \textit{Rademacher Complexity} of the function class $\mathcal{G}$ operating on $H$.
    \item Lower complexity tightens the generalization bound of the imputation error $\delta^2$ via standard Rademacher bounds, thereby directly minimizing the \textbf{Variance Term} in Eq. (\ref{eq:main_bound_result}).
\end{enumerate}


\section{Detailed Proofs and Derivations}
\label{sec:detailed_proofs}

This section provides the rigorous mathematical derivations for the theorems presented in Appendix \ref{sec:theoretical_main}. We explicitly detail the concentration inequalities and the chaining technique used to bound the model complexity.

\subsection{Proof of Theorem \ref{thm:generalization_bound}}
\label{prf:generalization_bound}

The proof relies on decomposing the true risk into bias and variance components and applying concentration inequalities for sequential data.

\subsubsection{Bias Analysis}
\begin{lemma}[Bias of the DR Estimator]
\label{lemma:bias_proof}
The bias of the DR estimator is bounded by the product of propensity error and imputation error:
\begin{equation}
    \left| \mathcal{R}_\text{true} - \mathbb{E}_O [\hat{\mathcal{R}}_\text{DR}] \right| \leq \frac{1}{N} \sum_{(t,c) \in \mathcal{D}_\text{full}} \left| \Delta_{t+1,c} \cdot \delta_{t+1,c} \right|,
\end{equation}
\end{lemma}

\begin{proof}
Recall that $\mathbb{E}[o_{t,c}] = p_{t,c}$. Taking the expectation of the estimator over observation indicators $O$:
\begin{equation}
\begin{split}
    \mathbb{E}_O [\hat{\mathcal{R}}_\text{DR}] &= \frac{1}{N} \sum_{(t,c) \in \mathcal{D}_\text{full}} \left( \hat{e}_{t+1,c} + \frac{p_{t+1,c}}{\hat{p}_{t+1,c}}(e_{t+1,c} - \hat{e}_{t+1,c}) \right) \\&= \frac{1}{N} \sum_{(t,c) \in \mathcal{D}_\text{full}} \left( \hat{e}_{t+1,c} + \frac{p_{t+1,c}}{\hat{p}_{t+1,c}}\delta_{t+1,c} \right).   
\end{split}
\end{equation}
Substituting the true risk definition $\mathcal{R}_\text{true} = \frac{1}{N} \sum e_{t+1,c} = \frac{1}{N} \sum (\hat{e}_{t+1,c} + \delta_{t+1,c})$ into the bias gap:
\begin{equation}
\begin{aligned}
    &\text{Bias} \\&= \left| \frac{1}{N} \sum_{(t,c)} (\hat{e}_{t+1,c} + \delta_{t+1,c}) - \frac{1}{N} \sum_{(t,c)} \left( \hat{e}_{t+1,c} + \frac{p_{t+1,c}}{\hat{p}_{t+1,c}}\delta_{t+1,c} \right) \right| \\
    &= \left| \frac{1}{N} \sum_{(t,c)} \delta_{t+1,c} \left( 1 - \frac{p_{t+1,c}}{\hat{p}_{t+1,c}} \right) \right| \\&= \left| \frac{1}{N} \sum_{(t,c)} \delta_{t+1,c} \left( \frac{\hat{p}_{t+1,c} - p_{t+1,c}}{\hat{p}_{t+1,c}} \right) \right|.
\end{aligned}
\end{equation}
Using the definition of relative propensity error $\Delta_{t,c}$ and the triangle inequality $|\sum x| \le \sum |x|$, we obtain the lemma.
\end{proof}

\subsubsection{Variance Analysis via Azuma-Hoeffding}
\begin{lemma}[Concentration for a Single Hypothesis]
\label{lemma:variance_proof}
For a fixed hypothesis $h$, with probability at least $1-\eta'$, the deviation is bounded by:
\begin{equation}
    \left| \hat{\mathcal{R}}_\text{DR} - \mathbb{E}[\hat{\mathcal{R}}_\text{DR}] \right| \leq \sqrt{ \frac{\ln(2/\eta')}{2 N^2} \sum_{(t,c) \in \mathcal{D}_\text{full}} \left( \frac{\delta_{t+1,c}}{\hat{p}_{t+1,c}} \right)^2 }.
\end{equation}
\end{lemma}
\begin{proof}
Let $Z_{t,c} = \hat{e}_{t,c} + \frac{o_{t,c}\delta_{t,c}}{\hat{p}_{t,c}}$ be the random variable for each sample. Since student interactions are sequential, the estimation errors conditional on history form a \textbf{Martingale Difference Sequence (MDS)}. 
The range of each term $Z_{t+1,c}$ is determined by the random variable $o_{t+1,c} \in \{0,1\}$:

\begin{equation}
\begin{split}
    &\text{Range}(Z_{t+1,c}) \\&= \left| \left( \hat{e}_{t+1,c} + \frac{1 \cdot \delta_{t+1,c}}{\hat{p}_{t+1,c}} \right) - \left( \hat{e}_{t+1,c} + \frac{0 \cdot \delta_{t+1,c}}{\hat{p}_{t+1,c}} \right) \right| \\&= \left| \frac{\delta_{t+1,c}}{\hat{p}_{t+1,c}} \right|.
\end{split}
\end{equation}
Applying the \textbf{Azuma-Hoeffding inequality} ~\cite{azuma1967weighted} for sum of bounded martingale differences:
\begin{equation}
\begin{split}
    P\left( |\hat{\mathcal{R}}_\text{DR} - \mathbb{E}[\hat{\mathcal{R}}_\text{DR}]| \geq \epsilon \right) &\leq 2 \exp \left( \frac{-2 (N\epsilon)^2}{\sum_{(t,c)} (\text{Range}_{t+1,c})^2} \right) \\&= 2 \exp \left( \frac{-2 N^2 \epsilon^2}{\sum (\delta_{t+1,c} / \hat{p}_{t+1,c})^2} \right).    
\end{split}
\end{equation}
Setting the right-hand side to $\eta'$ and solving for $\epsilon$ yields the stated bound.
\end{proof}
\subsubsection{Completing Proof of Theorem \ref{thm:generalization_bound}}

We then combine the Bias analysis and the Variance concentration to derive the final generalization bound for the hypothesis space $\mathcal{H}$.

\begin{proof}
First, we decompose the true risk $\mathcal{R}_{true}$ relative to our DR estimator $\hat{\mathcal{R}}_\text{DR}$. By the triangle inequality, the error can be split into a Bias term and a Variance term:
\begin{equation}
    \mathcal{R}_{true} \leq \hat{\mathcal{R}}_\text{DR} + \underbrace{\left| \mathcal{R}_{true} - \mathbb{E}[\hat{\mathcal{R}}_\text{DR}] \right|}_{\text{Bias}} + \underbrace{\left| \mathbb{E}[\hat{\mathcal{R}}_\text{DR}] - \hat{\mathcal{R}}_\text{DR} \right|}_{\text{Variance}}.
\end{equation}
The \textbf{Bias} term has been bounded in the previous theorem. Our remaining task is to bound the \textbf{Variance} term uniformly over the entire hypothesis space.

Lemma \ref{lemma:variance_proof} provides a concentration bound for a \textit{single fixed} hypothesis $h$. However, since the learning algorithm selects a hypothesis $\hat{h}$ based on the data, we require the bound to hold for the worst-case scenario over all $h \in \mathcal{H}$.

Let $\mathcal{E}(h)$ denote the event that the deviation for hypothesis $h$ exceeds $\epsilon$. According to Lemma \ref{lemma:variance_proof}, for any fixed $h$, $P(\mathcal{E}(h)) \leq \eta'$. We apply the \textbf{Union Bound} to control the probability that \textit{any} hypothesis in $\mathcal{H}$ violates the bound:

\begin{equation}
\begin{aligned}
    P\left( \max_{h \in \mathcal{H}} \left| \hat{\mathcal{R}}_\text{DR}(h) - \mathbb{E}[\hat{\mathcal{R}}_\text{DR}(h)] \right| \ge \epsilon \right) 
    &= P\left( \bigcup_{h \in \mathcal{H}} \mathcal{E}(h) \right) \\
    &\le \sum_{h \in \mathcal{H}} P(\mathcal{E}(h)) \\
    &\le \sum_{h \in \mathcal{H}} \eta' = |\mathcal{H}| \cdot \eta'.
\end{aligned}
\end{equation}

Finally, to ensure the bound holds with high probability $1-\eta$ over the entire space, we set the total failure probability to $\eta$:
\begin{equation}
    |\mathcal{H}| \cdot \eta' = \eta \implies \eta' = \frac{\eta}{|\mathcal{H}|}.
\end{equation}
We now substitute this adjusted $\eta'$ back into the bound for $\epsilon$ derived in Lemma \ref{lemma:variance_proof}. The logarithmic term transforms as follows:
\begin{equation}
    \ln(2/\eta') = \ln\left( \frac{2}{\eta / |\mathcal{H}|} \right) = \ln\left( \frac{2|\mathcal{H}|}{\eta} \right).
\end{equation}
Substituting this back into the variance expression:
\begin{equation}
    \text{Variance} \leq \sqrt{ \frac{\ln(2|\mathcal{H}|/\eta)}{2 N^2} \sum_{(t,c) \in \mathcal{D}_\text{full}} \left( \frac{\delta_{t,c}}{\hat{p}_{t,c}} \right)^2 }.
\end{equation}
Combining this uniform variance bound with the bias bound completes the proof of Theorem \ref{thm:generalization_bound}. 
\end{proof}

\subsection{Proof of Theorem \ref{thm:smoothness_main}}

To prove Theorem \ref{thm:smoothness_main}, we utilize the Chaining technique to derive Dudley's entropy integral bound for the Rademacher complexity ~\cite{Chaining}.

\subsubsection{Deriving Dudley's Entropy Integral}
\begin{proposition}
The empirical Rademacher complexity is bounded by:
\begin{equation}
    \mathfrak{R}_T(\mathcal{G} \circ H) \leq \frac{12}{\sqrt{T}} \int_0^D \sqrt{\log \mathcal{N}(\epsilon, \mathcal{G}, d_T)} d\epsilon,
\end{equation}
where $D$ is the diameter of the function class.
\end{proposition}

\begin{proof}
We employ the Chaining Technique ~\cite{Chaining} to bound the empirical Rademacher complexity. The proof proceeds by constructing a sequence of approximations for the function class $\mathcal{G}$ at increasingly fine scales and bounding the complexity of the differences between these scales.

\textbf{Construction and Decomposition.} 
Let $\epsilon_0 = D$ be the diameter of the function class. We define a dyadic sequence of scales $\epsilon_j = 2^{-j} \epsilon_0$ for $j \geq 1$. For each scale $j$, let $V_j$ be a minimal $\epsilon_j$-cover of $\mathcal{G}$ with respect to the empirical metric $d_T$. For any $g \in \mathcal{G}$, let $\pi_j(g)$ denote the projection of $g$ onto $V_j$ (i.e., the nearest element in $V_j$). Setting $\pi_0(g) = 0$, and noting that $\pi_j(g)$ converges to $g$ as $j \to \infty$, we can decompose any function $g$ as a telescoping sum:
\begin{equation}
    g(h_t) = \sum_{j=1}^{\infty} \left( \pi_j(g)(h_t) - \pi_{j-1}(g)(h_t) \right).
\end{equation}
By the linearity of the Rademacher complexity and the sub-additivity of the supremum, the total complexity is bounded by the sum of complexities at each scale:
\begin{equation}
\begin{aligned}
    &\mathfrak{R}_T(\mathcal{G} \circ H) \\&\leq \sum_{j=1}^{\infty} \mathbb{E}_{\boldsymbol{\sigma}} \left[ \sup_{g \in \mathcal{G}} \frac{1}{T} \sum_{t=1}^{T} \sigma_t \underbrace{\left( \pi_j(g)(h_t) - \pi_{j-1}(g)(h_t) \right)}_{\Delta_j(g)(h_t)} \right].
\end{aligned}
\end{equation}

\textbf{Bounding Increments via Massart's Lemma.} 
We now focus on the difference function $\Delta_j(g) = \pi_j(g) - \pi_{j-1}(g)$. The cardinality of the set of all such differences is at most $|V_j| \cdot |V_{j-1}| \leq |V_j|^2$. To apply Massart's Lemma, we first bound the Euclidean norm of the vector $\boldsymbol{\Delta}_j(g) = (\Delta_j(g)(h_1), \dots, \Delta_j(g)(h_T))$. Using the definition of the metric $d_T$ and the triangle inequality:
\begin{equation}
\begin{aligned}
    \|\boldsymbol{\Delta}_j(g)\|_2 &= \sqrt{T} \cdot \|\pi_j(g) - \pi_{j-1}(g)\|_{d_T} \\
    &\leq \sqrt{T} \left( \|\pi_j(g) - g\|_{d_T} + \|g - \pi_{j-1}(g)\|_{d_T} \right) \\
    &\leq \sqrt{T} (\epsilon_j + \epsilon_{j-1}) = 3\epsilon_j \sqrt{T},
\end{aligned}
\end{equation}
where we used $\epsilon_{j-1} = 2\epsilon_j$. Applying \textbf{Massart's Lemma} ~\cite{Chaining} to the $j$-th term yields:
    \textbf{Bounding Increments via Massart's Lemma.} 
    We now focus on the difference function class $\mathcal{F}_j = \{ \pi_j(g) - \pi_{j-1}(g) : g \in \mathcal{G} \}$. The cardinality of this set is bounded by the product of the covering numbers: $|\mathcal{F}_j| \leq |V_j| \cdot |V_{j-1}| \leq |V_j|^2$.
    
    To apply Massart's Lemma, we identify the two required parameters:
    \begin{itemize}
        \item \textbf{Radius ($R$):} As derived above, the Euclidean norm is bounded by $R = \sup_{f \in \mathcal{F}_j} \|f\|_2 \leq 3\epsilon_j \sqrt{T}$.
        \item \textbf{Cardinality ($K$):} The number of functions is $K = |\mathcal{F}_j| \leq \mathcal{N}(\epsilon_j)^2$.
    \end{itemize}
    
    Substituting these into the lemma and simplifying the expectation term:
    \begin{equation}
    \begin{aligned}
        \ \ \ \ \ \ \ \ \ \ \ \ \ \ \ \ \ \ \ \ \ \ \ 
        &\mathbb{E}_{\boldsymbol{\sigma}} \left[ \sup_{g} \frac{1}{T} \sum_{t=1}^T \sigma_t \Delta_j(g)(h_t) \right] 
        \\&= \frac{1}{T} \mathbb{E}_{\boldsymbol{\sigma}} \left[ \sup_{f \in \mathcal{F}_j} \sum_{t=1}^T \sigma_t f(h_t) \right] \\
        &\leq \frac{1}{T} \cdot \underbrace{3\epsilon_j \sqrt{T}}_{\text{Radius } R} \cdot \sqrt{2 \log (\underbrace{|V_j|^2}_{\text{Card. } K})} \\
        &= \frac{3\epsilon_j}{\sqrt{T}} \sqrt{2 \cdot 2 \log |V_j|} \quad \\
        &= \frac{6\epsilon_j}{\sqrt{T}} \sqrt{\log \mathcal{N}(\epsilon_j, \mathcal{G}, d_T)}.
    \end{aligned}
    \end{equation}

Finally, we aggregate these bounds. Using the relation $\epsilon_j = 2(\epsilon_j - \epsilon_{j+1})$, we can interpret the sum as a Riemann approximation. Since the covering number is monotonic, we have $\sqrt{\log \mathcal{N}(\epsilon)} \geq \sqrt{\log \mathcal{N}(\epsilon_j)}$ for $\epsilon \in (\epsilon_{j+1}, \epsilon_j]$. This allows us to upper bound the discrete sum by Dudley's entropy integral:
\begin{equation}
\begin{aligned}
    \mathfrak{R}_T &\leq \sum_{j=1}^{\infty} \frac{6\epsilon_j}{\sqrt{T}} \sqrt{\log \mathcal{N}(\epsilon_j)} \\
    &= \frac{12}{\sqrt{T}} \sum_{j=1}^{\infty} (\epsilon_j - \epsilon_{j+1}) \sqrt{\log \mathcal{N}(\epsilon_j)} \\
    &\leq \frac{12}{\sqrt{T}} \sum_{j=1}^{\infty} \int_{\epsilon_{j+1}}^{\epsilon_j} \sqrt{\log \mathcal{N}(\epsilon)} \, d\epsilon \\
    &= \frac{12}{\sqrt{T}} \int_0^{\epsilon_0} \sqrt{\log \mathcal{N}(\epsilon, \mathcal{G}, d_T)} \, d\epsilon.
\end{aligned}
\end{equation}
This concludes the proof.
\end{proof}

\subsubsection{Linking Smoothness to Path Length and Covering Number}

\begin{lemma}[Path Length Bound]
Let the path length be \(\ell(H) = \sum_{t=1}^{T} \|h_t - h_{t-1}\|_2\). The covering number of the trajectory satisfies:
\begin{equation}
    \mathcal{N}(\delta, H, \|\cdot\|_2) \leq \frac{\ell(H)}{\delta} + 1 \leq \frac{\sqrt{T \cdot \mathcal{L}_{smooth}}}{\delta} + 1.
\end{equation}
\end{lemma}
\begin{proof}
The proof consists of two steps: relating path length to smoothness, and covering number to path length.
\paragraph{Algebraic Bound.} Let \(d_t = \|h_t - h_{t-1}\|_2\). By definition, \(\ell(H) = \sum d_t\) and \(\mathcal{L}_{smooth} = \sum d_t^2\). Applying the Cauchy-Schwarz inequality to vectors \(\mathbf{d}\) and \(\mathbf{1}\) of dimension \(T\):
\begin{equation}
\begin{split}
    \ell(H) &= \sum_{t=1}^{T} (d_t \cdot 1) \\&\leq \sqrt{\sum_{t=1}^{T} d_t^2} \cdot \sqrt{\sum_{t=1}^{T} 1^2} \\&= \sqrt{\mathcal{L}_{smooth}} \cdot \sqrt{T}.    
\end{split}
\end{equation}
\paragraph{Geometric Bound.} Consider the trajectory \(H\) as a sequential curve of length \(L = \ell(H)\). Construct a covering greedily: start at \(h_1\), and place the next ball center at the first point along the curve at least distance \(\delta\) away. If \(k\) balls are needed, the total length must satisfy \(L \geq (k-1)\delta\), implying \(k \leq L/\delta + 1\).
\end{proof}

\subsubsection{Completing Proof of Theorem \ref{thm:smoothness_main}}
\begin{proof}
Assuming $g \in \mathcal{G}$ is $L_g$-Lipschitz and bounded by $M$, the covering number of the function space is controlled by the covering number of the input trajectory $H$. Specifically, an $\epsilon$-cover of $\mathcal{G}$ under the $L_\infty$ metric can be constructed from a $(\delta = \epsilon/L_g)$-cover of $H$.
Considering the values in range $[-M, M]$, we have the standard bound:
\begin{equation}
    \log \mathcal{N}(\epsilon, \mathcal{G}, d_T) \leq \mathcal{N}\left(\frac{\epsilon}{L_g}, H, \|\cdot\|_2\right) \cdot \log \left( \frac{2M}{\epsilon} + 1 \right).
\end{equation}
Substituting the result from Lemma 3 into this inequality, and then plugging the logarithmic term into the Dudley integral derived, yields the final bound in Theorem \ref{thm:smoothness_main}. Since $\mathcal{L}_{smooth}$ is in the numerator, minimizing it minimizes the complexity bound. 
\end{proof}

\end{appendices}
\end{document}